%% file: align.tex
\documentclass[12pt,pdftex,noinfoline]{imsart}
\RequirePackage[OT1]{fontenc}
\usepackage{amsthm,amsbsy,amsmath,amsfonts,natbib,mathtools,amssymb}
\usepackage{mathrsfs}
\RequirePackage{hypernat}
\usepackage{graphicx}
\usepackage{verbatim}
\usepackage{times}
\usepackage{hyperref}
\usepackage[usenames,dvipsnames,svgnames,table]{xcolor}
\hypersetup{citecolor=MidnightBlue}
\hypersetup{linkcolor=MidnightBlue}
\hypersetup{urlcolor=MidnightBlue}
\usepackage{enumerate}
\usepackage{fullpage}
\usepackage[margin=1in,footskip=.40in]{geometry}
\usepackage{dsfont}

\usepackage{tikz}
    \usetikzlibrary{
        calc,
        decorations.pathmorphing,
        fadings,
        shadows,
        shapes.geometric,
        shapes.misc,
        arrows,
        }
\usepackage{tikz-cd}

\numberwithin{equation}{section}
\newtheorem{theorem}{Theorem}[section]
\newtheorem{lemma}[theorem]{Lemma}
\newtheorem{corollary}[theorem]{Corollary}
\newtheorem{proposition}[theorem]{Proposition}
\theoremstyle{remark}
\theoremstyle{definition}
\theoremstyle{assumption}
\newtheorem{assumption}[theorem]{Assumption}
\newtheorem{definition}[theorem]{Definition}

\def\conditionA{Condition $A$}
\def\conditionAp{Condition $\tilde A$}
\def\conditionB{Condition $B$}

\let\hat\widehat

\let\tilde\widetilde

\input{macros}

\begin{document}

\setlength{\parskip}{0.5em}
\begin{frontmatter}
\runtitle{Convergence and Alignment}
\title{Convergence and Alignment of Gradient Descent\\ with Random Backpropagation Weights}
\begin{aug}
\vskip15pt
\address{
\vbox{
\begin{tabular}{c}
  {\normalsize\rm\bfseries Ganlin Song, Chris Xu, and John Lafferty}\\[10pt]
\end{tabular}
\vskip-2pt
\begin{tabular}{c}
Department of Statistics and Data Science\\
Yale University\\[10pt]
\end{tabular}
\\[10pt]
\today
\vskip10pt
}
}
\end{aug}
\input{abstract}
\end{frontmatter}

\input{introduction}
\input{overview}
\input{convergence}
\input{alignment}
\input{simulations}
\input{discussion}

\input{acknowledge}

\setlength{\bibsep}{8pt plus 0.3ex}
\bibliographystyle{apalike}
\bibliography{feedback_alignment}

\clearpage
\appendix
\input{proof_conv}
\input{proof_align}
\input{lemmas}

\newpage

\end{document}

%% file: macros.tex
\usepackage{amsfonts}
\usepackage{amsmath}
\usepackage{amsthm}
\usepackage{authblk}
\usepackage{amssymb}
\usepackage{mathrsfs}
\usepackage{bm}
\usepackage{bbm}
\usepackage{mathtools}
\usepackage{fancyhdr}
\usepackage{algorithm}
\usepackage{algpseudocode}
\usepackage[frak=esstix]{mathalfa}
\usepackage{xspace}
\usepackage{xfrac}
\usepackage{afterpage}
\usepackage{framed}
\usepackage[inline]{enumitem}
\usepackage[colorlinks]{hyperref}
	\hypersetup{linkcolor=[RGB]{200,40,41}}
	\hypersetup{citecolor=[RGB]{66,113,174}}
	\hypersetup{urlcolor=[RGB]{231,197,71}}
\usepackage[acronym,smallcaps,nowarn]{glossaries}
\usepackage{lmodern}

\usepackage{natbib}
\usepackage{upgreek}
\usepackage{booktabs}
\usepackage{subcaption}

\usepackage[svgnames]{xcolor}
    \definecolor{notecolor}{RGB}{137,89,168}
    \definecolor{quotecolor}{RGB}{66,113,174}
    \definecolor{warningcolor}{RGB}{249,145,87}
    \definecolor{YaleBlue}{RGB}{0, 53, 107}
    \definecolor{YaleBlueLight}{RGB}{40,109,192}
    \definecolor{YaleBlueLighter}{RGB}{99,170,255}
    \definecolor{YaleGreen}{RGB}{95, 113, 45}
    \definecolor{YaleOrange}{RGB}{189, 83, 25}
\usepackage[all]{hypcap}
\usepackage{seqsplit}
\usepackage{xstring}

\usepackage{graphicx}

\usepackage[capitalize]{cleveref}
\crefname{ineqn}{Ineq.}{Ineqs.}
\creflabelformat{ineqn}{(#2{\upshape#1}#3)}

\newcommand{\cf}{\emph{cf.}\!\xspace}
\newcommand{\ie}{\emph{i.e.}\xspace}

\newcommand{\iid}{i.i.d.\!\xspace}

\newcommand\transpose{^{\mathpalette\raiseT{\scriptstyle\intercal}}}
\newcommand\raiseT[2]{\raisebox{0.2ex}{$#1#2$}}
\DeclarePairedDelimiterX{\innerDelim}[2]{\langle}{\rangle}{#1, #2}

\newcommand{\reals}{{\mathbb{R}}}

\newcommand{\eps}{\varepsilon}

\newcommand{\RR}{\mathbb{R}}

\newcommand{\Rpd}{{\mathbb{R}^{p\times d}}}
\newcommand{\Rmn}{{\mathbb{R}^{m\times n}}}

\newcommand{\Rd}{{\mathbb{R}^{d}}}
\newcommand{\Rp}{{\mathbb{R}^{p}}}

\newcommand{\Rn}{{\mathbb{R}^n}}

\DeclareMathOperator{\prob}{\mathbb{P}}

\makeatletter
\@tfor\next:=ABCDEFGHIJKLMNOPQRSTUVWXYZ\do{%
  \def\command@factory#1{%
    \expandafter\def\csname cal#1\endcsname{{\mathcal{#1}}}
  }
 \expandafter\command@factory\next
}
\makeatother

\makeatletter
\@tfor\next:=ABCDEFGHIJKLMNOPQRSTUVWXYZ\do{%
  \def\command@factory#1{%
    \expandafter\def\csname #1#1\endcsname{{\mathbb{#1}}}
  }
 \expandafter\command@factory\next
}
\makeatother

\newcommand{\defeq}{\coloneqq}

\numberwithin{equation}{section}

\newcommand{\Var}{\mathbb{V}\mathrm{ar}}

\newcommand{\R}{\mathbb{R}}
\newcommand{\E}{\mathbb{E}}
\newcommand{\I}{\mathbb{I}}

\newcommand\Loss{\mathcal{L}}

\newcommand{\dfa}{\delta_\mathrm{FA}}
\newcommand{\dbp}{\delta_\mathrm{BP}}

%% file: abstract.tex
\begin{abstract}
Stochastic gradient descent with backpropagation is the workhorse of artificial neural networks. It has long been recognized that backpropagation fails to be a biologically plausible algorithm. Fundamentally, it is a non-local procedure---updating one neuron's synaptic weights requires knowledge of synaptic weights or receptive fields of downstream neurons. This limits the use of artificial neural networks as a tool for understanding the biological principles of information processing in the brain. Lillicrap et al.~(2016) propose a more biologically plausible ``feedback alignment'' algorithm that uses random and fixed backpropagation weights, and show promising simulations and analysis. In this paper we study the mathematical properties of the feedback alignment procedure by analyzing convergence and alignment for two-layer networks under squared error loss. In the overparameterized setting, we prove that the error converges to zero exponentially fast, and also that regularization is necessary in order for the  parameters to become aligned with the random backpropagation weights. Simulations are given that are consistent with this analysis and suggest further generalizations. These results contribute to our understanding of how biologically plausible algorithms might carry out weight learning in a manner different from Hebbian learning, with performance that is comparable with the full non-local backpropagation algorithm.
\end{abstract}

%% file: introduction.tex
\section{Introduction}

The roots of artificial neural networks draw inspiration from networks of biological neurons \citep{pdp,elman,medler}. Grounded in simple abstractions of membrane potentials and firing, neural networks are increasingly being employed as a computational tool for better understanding the biological principles of information processing in the brain; examples include \cite{ilker1} and \cite{yamins2}. Even when full biological fidelity is not required, it can be useful to better align the computational abstraction with neuroscience principles.

Stochastic gradient descent has been a workhorse of artificial neural networks. Conveniently, calculation of gradients can be carried out using the backpropagation algorithm, where reverse mode automatic differentiation provides a powerful way of computing the derivatives for general architectures \citep{rumelhart:86}.
Yet it has long been recognized that backpropagation fails to be a biologically plausible algorithm. Fundamentally, it is a non-local procedure---updating the weight between a presynaptic and postsynaptic neuron requires knowledge of the weights between the postsynaptic neuron and other neurons. No known biological mechanism exists for propagating information in this manner. This limits the use of artificial neural networks as a tool for understanding learning in the brain.

A wide range of approaches have been explored as a potential basis for learning and synaptic plasticity. Hebbian learning is the most fundamental procedure for adjusting weights, where
repeated stimulation by a presynaptic neuron that results in the subsequent
firing of the postsynapic neuron will result in an increased strength in the connection
between the two cells \citep{hebb1,paulsen}. Several variants of Hebbian learning, some making connections to principal components analysis, have been proposed
\citep{oja,sejnowski1,sejnowski2}.
In this paper, our focus is on a formulation of \cite{lillicrap2016random} based on random backpropagation weights that are fixed during the learning process, called the ``feedback alignment" (FA) algorithm.
\cite{lillicrap2016random} show that the model can still learn from data, and observe the interesting phenomenon that the error signals propagated with the forward weights align with those propagated with fixed random backward weights during training. Direct feedback alignment (DFA) \citep{nokland2016direct} extends FA by adding skip connections to send the error signals directly to each hidden layer, allowing parallelization of weight updates. Empirical studies given by \citet{launay2020direct} show that DFA can be successfully applied to train a number of modern deep learning models, including transformers. Based on DFA, \citet{frenkel2021learning} proposes direct the random target projection (DRTP) algorithm that trains the network weights with a random projection of the target vector instead of the error, and shows alignment for linear networks.
Related proposals, including methods based on the use of differences of neuron activities, have been made in a series of recent papers \citep{akrout,bellec,lillicrap2020backpropagation}. A comparison of some of these methods is made by \cite{bartunov}.

The use of random feedback weights, which are not directly tied to the forward weights, removes issues of non-locality. However, it is not clear under what conditions optimization of error and learning can be successful.
While \mbox{\citet{lillicrap2016random}} give suggestive simulations and some analysis for the linear case, it has been an open problem to explain the behavior of this algorithm for training the weights of a neural network.
In this paper, we study the mathematical properties of the feedback alignment procedure by analyzing convergence and alignment for two-layer networks under squared error loss. In the overparameterized setting, we prove that the error converges to zero exponentially fast. We also show, unexpectedly, that the parameters become aligned with the random backpropagation weights only when regularization is used. Simulations are given that are consistent with this analysis and suggest further generalizations. The following section gives further background and an overview of our results.

%% file: overview.tex
\section{Problem Statement and Overview of Results}

In this section we provide a formulation of the backpropagation
algorithm to establish notation and the context for our analysis. We then
formulate the feedback aligment algorithm that uses random backpropation weights.
A high-level overview of our results is then presented, together with
some of the intuition and proof techniques behind these results; we also contrast with what was known previously.

We mainly consider two-layer neural networks in the regression setting, specified by a family of functions  $f:\Rd \to \RR$ with input dimension $d$, sample size $n$, and $p$ neurons in the hidden layer. For an input $x\in\Rd$, the network outputs
\begin{align}\label{eqn:nonlinear-network}
    f(x) = \frac{1}{\sqrt p}\sum_{r=1}^p\beta_r\psi(w_r\transpose x)= \frac{1}{\sqrt p}\beta\transpose\psi(Wx),
\end{align}
where $W = (w_1,...,w_p)\transpose\in\Rpd$ and $\beta = (\beta_1,...,\beta_p)\transpose\in\Rp$ represent the feed-forward weights in the first and second layers, and $\psi$ denotes an element-wise activation function. The scaling by $\sqrt{p}$ is simply for convenience in the analysis.

Given $n$ input-response pairs $\{(x_i,y_i)\}_{i=1}^n$, the training objective is to minimize the squared error
\begin{equation}\label{eqn:squared-loss}
    \Loss(W,\beta) = \frac{1}{2}\sum_{i=1}^n \big(y_i - f(x_i)\big)^2.
\end{equation}
Standard gradient descent attempts to minimize \eqref{eqn:squared-loss} by updating the feed-forward weights following gradient directions according to
\begin{align*}
    \beta_r(t+1) &= \beta_r(t)-\eta\frac{\partial\Loss}{\partial \beta_r}(W(t),\beta(t)) \quad\\ w_r(t+1) &= w_r(t)-\eta\frac{\partial\Loss}{\partial w_r}(W(t),\beta(t)),
\end{align*}
for each $r\in[p]$, where $\eta>0$ denotes the step size. We initialize $\beta(0)$ and $w_r(0)$ as standard Gaussian vectors. We introduce the notation $f(t), e(t)\in \R^n$, with $f_i(t) = f(x_i)$ denoting the network output on input $x_i$ when the weights are $W(t)$ and $\beta(t)$, and $e_i(t) = y_i-f_i(t)$ denoting the corresponding prediction error or residual. With this notation,
the gradients are expressed as
\begin{equation*}
    \frac{\partial\Loss}{\partial \beta_r} = \frac{1}{\sqrt p}\sum_{i=1}^n e_i\psi(w_r\transpose x_i), \quad
    \frac{\partial\Loss}{\partial w_r} = \frac{1}{\sqrt p} \sum_{i=1}^n e_i \beta_r\psi'(w_r\transpose x_i)x_i.
\end{equation*}
Here it is seen that the the gradient of the first-layer weights $\frac{\partial \Loss}{\partial w_r}$ involves not only the local input $x_i$ and the change in
the response of the $r$-th neuron, but also the backpropagated error signal $e_i\beta_r$.
The appearance of $\beta_r$ is, of course, due to the chain rule; but in effect it requires that the forward weights between layers are identical to the backward weights under error propagation. There is no evidence of biological mechanisms that would enable such ``synaptic symmetry.''

\begin{figure*}[t]
  \begin{tabular}{cc}
    \begin{minipage}{.4\textwidth}
    \tikzset{circ/.style={circle, draw, fill=YaleBlueLight!40, scale=0.7,font=\large}}
    \tikzset{rect/.style={rectangle, rounded corners, draw, fill=YaleBlueLight!40, scale=0.7,font=\Large}}
    \begin{tikzpicture}[xscale=0.65,yscale=0.9]
    
    \node (L21) at (4,-1.5) [circ]{$x_d$};
    \node (L20) at (4,-.65) {$\vdots$};
    \node (L20) at (4,.85) {$\vdots$};
    \node (L20) at (4,0) [circ]{$x_i$};
    \node (L24) at (4,1.5) [circ]{$x_1$};
    
    \node (L30) at (8,0) [circ]{$h_r$};
    \node (L31) at (8,-.9) {$\vdots$};
    \node (L34) at (8,1.1) {$\vdots$};
    \node (L37) at (8,-2) [circ]{$h_p$};
    \node (L38) at (8,2) [circ]{$h_1$};
    
    \draw[->, >=latex', shorten >=2pt,color=black,thick] (L20) to node [auto] {} (L30);
    \draw[->, >=latex', shorten >=2pt,color=lightgray,thick] (L20) to node [auto] {} (L37);
    \draw[->, >=latex', shorten >=2pt,color=lightgray,thick] (L20) to node [auto] {} (L38);
    
    \draw[->, >=latex', shorten >=2pt,color=black,thick] (L21) to node [auto] {} (L30);
    \draw[->, >=latex', shorten >=2pt,color=lightgray,thick] (L21) to node [auto] {} (L37);
    \draw[->, >=latex', shorten >=2pt,color=lightgray,thick] (L21) to node [auto] {} (L38);

    \draw[->, >=latex', shorten >=2pt,color=black,thick] (L24) to node [auto] {} (L30);
    \draw[->, >=latex', shorten >=2pt,color=lightgray,thick] (L24) to node [auto] {} (L37);
    \draw[->, >=latex', shorten >=2pt,color=lightgray,thick] (L24) to node [auto] {} (L38);
    
    \node at (5.5,-2.0) {$W$};
    
    \node (L40) at (11,0) [rect]{$f(x)$};
    
    \draw[->, >=latex', shorten >=2pt,color=black,thick] (L30) to node [auto] {$\beta_r$} (L40);
    \draw[->, >=latex', shorten >=2pt,color=lightgray,thick] (L37) to node [auto] {} (L40);
    \draw[->, >=latex', shorten >=2pt,color=lightgray,thick] (L38) to node [auto] {} (L40);
    
    \draw[<-, >=latex', shorten >=2pt, shorten <=2pt, bend right=50, dashed, color=YaleBlueLight, thick] 
        (L30) to node[auto, swap] {$b_r$}(L40);
    
    \end{tikzpicture}
    \end{minipage}
    \begin{minipage}{.47\textwidth}
    \begin{algorithm}[H]
    \centering
    \caption{Feedback Alignment}\label{algo:fa}
        \begin{algorithmic}[1]
            \Require Dataset $\{(x_i,y_i)\}_{i=1}^n$, step size $\eta$
            \State {\bf initialize} $W$, $\beta$ and $b$ as Gaussian
            \While{not converged}
                \State $\beta_r \gets \beta_r - \frac{\eta}{\sqrt p} \sum_{i=1}^n e_i \psi(w_r\transpose x_i)$
                \State $w_r \gets w_r - \frac{\eta}{\sqrt{p}} \sum_{i=1}^n e_i b_r\psi'(w_r\transpose x_i)x_i$
                \State for $r\in[p]$
            \EndWhile
        \end{algorithmic}
    \end{algorithm}%
    \end{minipage}
    \\[1.05in]

  \end{tabular}
\caption{Standard backpropagation updates the first layer weights for a hidden node $r$ with the second layer feedforward weight $\beta_r$. We study the procedure where the error is backpropagated instead using a fixed, random weight $b_r$.}
\label{fig:algo}
\end{figure*}

In the \textit{feedback alignment} procedure of \citep{lillicrap2016random},
when updating the weights $w_r$, the error signal is weighted, and propagated backward, not by the second layer feedforward weights $\beta$, but rather by a random set of weights $b\in\reals^p$ that are fixed during the course of training. Equivalently, the gradients for the first layer are
replaced by the terms
\begin{align}\label{eqn:alignment-update}
  \widetilde{\frac{\partial\Loss}{\partial w_r}}   = \frac{1}{\sqrt{p}} \sum_{i=1}^n e_i b_r\psi'(w_r\transpose x_i)x_i.
\end{align}
Note, however, that this update rule does not correspond to the gradient with
respect to a modified loss function. The use of a random weight $b_r$ when updating
the first layer weights $w_r$ does not violate locality, and could conceivably be implemented by biological mechanisms; we refer to \cite{lillicrap2016random,bartunov,lillicrap2020backpropagation} for further discussion. A schematic of the relationship between the two algorithms is shown in Figure~\ref{fig:algo}.

We can now summarize the main results and contributions of this paper. Our first result shows that the error converges to zero when using random backpropagation weights.

\begin{itemize}
  \item Under Gaussian initialization of the parameters, if the model is sufficiently over-parameterized with $p\gg n$, then the error converges to zero linearly. Moreover, the parameters satisfy $\|w_r(t) - w_r(0) \| = \widetilde O\bigl(\frac{n}{\sqrt{p}}\bigr)$
    and $|\beta_r(t) - \beta_r(0) | = \widetilde O\bigl(\frac{n}{\sqrt{p}}\bigr)$.
\end{itemize}
The precise assumptions and statement of this result are given in \cref{thm:nonliner_conv}. The proof
shows in the over-parameterized regime that the weights only change
by a small amount. While related to results for standard gradient descent,
new methods are required because the ``effective kernel'' is not positive semi-definite.

We next turn to the issue of alignment of the second layer parameters $\beta$ with the random backpropagation weights $b$. Such alignment was first observed in the original simulations of \cite{lillicrap2016random}. With $h\in \R^p$ denoting the hidden layer of the two-layer network,  the term $\dbp(h) \defeq \frac{\partial \Loss}{\partial h} = \frac{1}{\sqrt{p}} \beta \sum_{i=1}^n e_i$ represents
how the error signals $e_i$ are sent backward to update the feed-forward weights.
With the use of random backpropagation weights, the error is instead propagated backward as $\dfa(h) = \frac{1}{\sqrt{p}} b \sum_{i=1}^n e_i$.

\citet{lillicrap2016random} notice a decreasing angle between $\dbp(h)$ and $\dfa(h)$ during training, which is a sufficient condition to ensure that the algorithm converges.
In the case of $k$-way classification, the last layer has $k$ nodes,
$\beta$ and $b$ are $p\times k$ matrices, and each error term $e_i$ is a $k$-vector.
In the regression setting, $k=1$ so the angle between
$\dbp(h)$ and $\dfa(h)$ is the same as the angle between $\beta$ and $b$.
Intuitively, the possibility for alignment is seen in the fact that while the updates for $W$ use the error weighted by the random weights $b$, the updates for $\beta$ indirectly involve $W$, allowing for the possibility that dependence on $b$ will be introduced into $\beta$.

Our first result shows that, in fact, alignment will \textit{not} occur in the over-parameterized setting. (So, while the error may still converge, ``feedback alignment'' may be a bit of a misnomer for the algorithm.)
\begin{itemize}
\item The cosine of the angle between
  the $p$-dimensional vectors $\dfa$ and $\dbp$ satisfies
  $$\cos\angle(\dfa, \dbp(t)) = \cos\angle(b, \beta(t)) = O\big(\frac{n}{\sqrt p}\big).$$

\end{itemize}
 However, we show that regularizing the parameters will cause
 $\dbp$ to align with $\dfa$ and therefore the parameters $\beta$ to align with $b$. Since $\beta(0)$ and $b$ are high dimensional Gaussian vectors, they are nearly orthogonal with high probability. The effect of regularization can be seen as shrinking the component of $\beta(0)$ in the parameters over time.  Our next result establishes this precisely in the linear case.
\begin{itemize}
\item Supposing that $\psi(u)=u$, then introducing a ridge penalty $\lambda(t) \|\beta\|^2$ where $\lambda(t) = \lambda$ for $t\leq T$ and $\lambda(t) = 0$ for $t > T$
on $\beta$  causes the parameters to align, with $\cos\angle(b, \beta(t)) \geq c > 0$ for sufficiently large $t$.
\end{itemize}
The technical conditions are given in \cref{thm:lin_align}.
Our simulations are consistent with this result, and also show alignment with a constant regularization $\lambda(t)\equiv \lambda$, for both linear and nonlinear activation functions. Finally, we complement this result by showing that convergence is preserved with regularization, for general activation functions. This is presented in \cref{thm:nonlinear_conv_reg}.

%% file: convergence.tex

\section{Convergence with Random Backpropagation Weights}

Due to the replacement of backward weights with the random backpropagation weights, there is no guarantee \emph{a priori} that the algorithm will reduce the squared error loss $\Loss$. \citet{lillicrap2020backpropagation} study the convergence on two-layer linear networks in a continuous time setting. Through the analysis of a system of differential equations on the network parameters, convergence to the true linear target function is shown, in the population setting of  arbitrarily large training data.
Among recent studies of over-parametrized networks under backpropagation, the neural tangent kernel (NTK) is heavily utilized to describe the evolution of the network during training \citep{jacot2018neural,chen2020deep}. For any neural network $f(x,\theta)$ with parameter $\theta$, the NTK is defined as
\begin{equation*}
	K_f(x,y) = \Big\langle \frac{\partial f (x,\theta)}{\partial \theta},\frac{\partial f (y,\theta)}{\partial \theta}\Big\rangle.
\end{equation*}
Given a dataset $\{(x_i,y_i)\}_{i=1}^n$, we can also consider its corresponding Gram matrix $K = (K_f(x_i,x_j))_{n\times n}$. \citet{jacot2018neural} show that in the infinite width limit, $K_f$ converges to a constant at initialization and does not drift away from initialization throughout training. In the over-parameterized setting, if the Gram matrix $K$ is positive definite, then $K$ will remain close to its initialization during training, resulting in linear convergence of the squared error loss \citep{du2018gradient,du2019gradient,gao2020model}.
For the two-layer network $f(x, \theta)$ defined in \eqref{eqn:nonlinear-network} with $\theta = (\beta,W)$, the kernel $K_f$ can be written in two parts, $G_f$ and $H_f$, which correspond to $\beta$ and $W$ respectively:
\begin{equation*}
K_f(x,y) = G_f(x, y) + H_f(x,y) \defeq \Big\langle \frac{\partial f (x,\theta)}{\partial \beta},\frac{\partial f (y,\theta)}{\partial \beta}\Big\rangle + \sum_{r=1}^p\Big\langle \frac{\partial f (x,\theta)}{\partial w_r},\frac{\partial f (y,\theta)}{\partial w_r}\Big\rangle.
\end{equation*}
Under the feedback alignment scheme with random backward weights $b$, $G_f$ remains the same as for standard backpropagation, while one of the gradient terms $\frac{\partial f}{\partial w_r}$ in $H_f$ changes to
$\widetilde{\frac{\partial f(x, \theta)}{\partial w_r}}   = \frac{1}{\sqrt{p}} b_r \psi'(w_r\transpose x)x,
$
with $H_f$ replaced by $H_f = \sum_{r=1}^p\Big\langle \widetilde{\frac{\partial f (x,\theta)}{\partial w_r}},\frac{\partial f (y,\theta)}{\partial w_r}\Big\rangle$.
As a result, $H_f$ is no longer positive semi-definite and close to $0$ at initialization if the network is over-parameterized. However, if $G = (G_f(x_i,x_j))_{n\times n}$ is positive definite and $H = (H_f(x_i,x_j))_{n\times n}$ remains small during training, we are still able to show that the loss $\Loss$ will converge to zero exponentially fast.

\begin{assumption}\label{assump:G}
Define the matrix $\overline{G} \in \R^{n\times n}$ with entries
$\overline{G}_{i,j} = \E_{w\sim \calN(0,I_p)}\psi(w\transpose x_i) \psi(w\transpose  x_j)$.
Then we assume that the minimum eigenvalue satisfies $\lambda_{\min}(\overline{G}) \geq \gamma$, where $\gamma$ is a positive constant.
\end{assumption}
\vskip5pt

\begin{theorem}\label{thm:nonliner_conv}
Let $W(0)$, $\beta(0)$ and $b$ have \iid standard Gaussian entries. Assume \textnormal{(1)} \cref{assump:G} holds, \textnormal{(2)} $\psi$ is smooth, $\psi$, $\psi'$ and $\psi''$ are bounded and \textnormal{(3)} $|y_i|$ and $\|x_i\|$ are bounded for all $i\in[n]$. Then there exists positive constants $c_1$, $c_2$, $C_1$ and $C_2$, such that for any $\delta\in(0,1)$, if $p \geq \max\left(C_1\frac{n^2}{\delta\gamma^2}, C_2\frac{n^4\log p}{\gamma^4}\right)$, then with probability at least $1-\delta$ we have that
\begin{equation}\label{eq:conv}
    \|e(t+1)\| \leq (1-\frac{\eta\gamma}{4})\|e(t)\|
\end{equation}
and
\begin{equation}
\label{eq:weights}
    \|w_r(t)-w_r(0)\| \leq c_1\frac{n\sqrt{\log p}}{\gamma\sqrt p}, \quad |\beta_r(t)-\beta_r(0)| \leq c_2\frac{n}{\gamma\sqrt p}
\end{equation}
for all $r\in[p]$ and $t>0$.
\end{theorem}

We note that the matrix $\overline{G}$ in \cref{assump:G} is the expectation of $G$ with respect to the random initialization, and is thus close to $\overline{G}$ due to concentration. To justify the assumption, we provide the following proposition, which states that \cref{assump:G} holds when the inputs $x_i$ are drawn independently from a Gaussian distribution. The proofs of \cref{thm:nonliner_conv,prop:positive-definiteness} are deferred to \cref{sec:appendix-convergence}.

\begin{proposition}\label{prop:positive-definiteness}
Suppose $x_1,...,x_n \overset{\iid}{\sim} \calN(0,I_d/d)$ and the activation function $\psi$ is sigmoid or tanh. If $d=\Omega(n)$, then \cref{assump:G} holds with high probability.
\end{proposition}

%% file: alignment.tex

\section{Alignment with Random Backpropagation Weights}\label{sec:alignment}

The most prominent characteristic of the feedback alignment algorithm is the phenomenon that the error signals propagated with the forward weights align with those propagated with fixed random backward weights during training. Specifically, if we denote $h\in \R^p$ to be the hidden layer of the network, then we write $\dbp(h) \defeq \frac{\partial \Loss}{\partial h}$ to represent the error signals with respect to the hidden layer that are backpropagated with the feed-forward weights and $\dfa(h)$ as the error signals computed with fixed random backward weights.
In particular, the error signals $\dbp(h)$ and $\dfa(h)$ for the two-layer network \eqref{eqn:nonlinear-network} are given by
\begin{equation*}
    \dbp(h) = \frac{1}{\sqrt{p}}\beta\sum_{i=1}^ne_i \quad\text{and}\quad \dfa(h) = \frac{1}{\sqrt{p}}b\sum_{i=1}^ne_i.
\end{equation*}
\citet{lillicrap2016random} notice a decreasing angle between $\dbp(h)$ and $\dfa(h)$ during training. We formalize this concept of alignment by the following definition.
\begin{definition}\label{def:alignment}
    We say a two-layer network \textit{aligns} with the random weights $b$ during training if there exists a constant $c>0$ and time $T_c$ such that  $\cos\angle(\dfa, \dbp(t)) = \cos\angle(b, \beta(t)) = \frac{\langle b, \beta(t)\rangle}{\|b\|\|\beta(t)\|} \geq c$ for all $t > T_c$.
\end{definition}

\subsection{Regularized feedback alignment}
Unfortunately, alignment between $\beta(t)$ and $b$ is not guaranteed for over-parameterized networks and the loss \eqref{eqn:squared-loss}. In particular,  we control the cosine value of the angle by inequalities \eqref{eq:weights} from Theorem \ref{thm:nonliner_conv}, \ie,
\begin{equation*}
    \Big|\cos\angle(b, \beta(t))\Big| \leq \frac{|\langle \frac{b}{\|b\|}, \beta(0)\rangle|+ \|\beta(t)- \beta(0)\|}{\|\beta(0)\|-\|\beta(t)-\beta(0)\|} = O\left(\frac{n}{\sqrt p}\right),
\end{equation*}
which indicates that $\beta(t)$ and $b$ become orthogonal as the network becomes wider. Intuitively, this can be understood as resulting from the parameters staying near their initializations during training when $p$ is large, where $\beta(0)$ and $b$ are almost orthogonal to each other. This motivates us to regularize the network parameters. We consider in this work the squared error loss with an $\ell_2$ regularization term on $\beta$:
\begin{equation}
\label{eqn:loss-with-reg}
\Loss(t, W, \beta) = \frac{1}{2}\sum_{i=1}^n\big(f(x_i)-y_i\big)^2 + \frac{1}{2}\lambda(t)\|\beta\|^2,
\end{equation}
where $\{\lambda(t)\}_{t=0}^\infty$ is a sequence of regularization rates, which defines a series of loss functions for different training steps $t$.
Thus, the update for $w_r$ remains the same and the
update for $\beta$ changes to
$$\beta_r(t+1) = (1-\lambda(t))\beta_r(t) - \frac{\eta}{\sqrt p} \sum_{i=1}^n e_i(t)\psi(w_r(t)\transpose x_i), \;\;\mbox{for $r\in[p]$.}$$
Comparing to Algorithm \ref{algo:fa}, an extra contraction factor $1-\lambda(t)$ is added in the update of $\beta(t)$, which doesn't affect the locality of the algorithm but helps the alignment by shrinking the component of $\beta(0)$ in $\beta(t)$.


Following Theorem \ref{thm:nonliner_conv}, we provide an error bound for regularized feedback alignment in Theorem \ref{thm:nonlinear_conv_reg}. Since regularization terms $\lambda(t)$ make additional contributions to the error $e(t)$ as well as to the kernel matrix $G$, an upper bound on $\sum_{t\geq 0}\lambda(t)$ is needed to ensure positivity of the minimal eigenvalue of $G$ during training, in order for the error $e(t)$ to be controlled. In particular, if there is no regularization, \ie, $\lambda(t)=0$ for all $t\geq 0$, then we recover exponential convergence for the error $e(t)$ as in Theorem \ref{thm:nonliner_conv}. The proof of Theorem \ref{thm:nonlinear_conv_reg} is also deferred to \cref{sec:appendix-convergence}.

\begin{theorem}
\label{thm:nonlinear_conv_reg}
Assume all the conditions from Theorem \ref{thm:nonliner_conv}. Assume $\sum_{t=0}^\infty \lambda(t) \leq  \tilde{S}_\lambda = \tilde{c}_{S}\frac{\gamma^2\sqrt{p}}{\eta n^2\sqrt{\log p}}$ for some constant $\tilde{c}_{S}$. Then there exist positive constants $C_1$ and $C_2$, such that for any $\delta\in(0,1)$, if $p \geq \max\bigl(C_1\frac{n^2}{\delta\gamma^2}, C_2\frac{n^4\log p}{\gamma^4}\bigr)$, then with probability at least $1-\delta$, we have
\begin{equation}
\label{eq:conv_reg}
    \|e(t+1)\| \leq \Bigl(1-\frac{\eta\gamma}{4}-\eta\lambda(t)\Bigr)\|e(t)\|+\lambda(t)\|y\|
\end{equation}
for all $t\geq 0$.
\end{theorem}

\subsection{Alignment analysis for linear networks}
In this section, we focus on the theoretical analysis of alignment for linear networks, which is equivalent to setting the activation function $\psi$ to the identity map. The loss function can be written as
\begin{equation*}
    \Loss(t,W,\beta) = \frac{1}{2}\big\|\frac{1}{\sqrt{p}}XW\transpose\beta-y\big\|^2+\frac{\lambda(t)}{2}\|\beta\|^2,
\end{equation*}
where $X = (x_1,\ldots,x_n)\transpose$; this is a form of over-parameterized ridge regression.
Before presenting our results on alignment, we first provide a linear version of Theorem \ref{thm:nonlinear_conv_reg} that adopts slightly different conditions.

\begin{theorem}
\label{thm:lin_conv}
Assume \textnormal{(1)}  $\|y\| = \Theta(\sqrt n)$, $\lambda_{\min}(XX\transpose)>\gamma$ and $\lambda_{\max}(XX\transpose)<M$ for some constants $M>\gamma>0$, and \textnormal{(2)} $\sum_{t=0}^\infty \lambda(t) \leq  S_\lambda = c_{S}\frac{\gamma\sqrt{\gamma p}}{\eta\sqrt{n}M}$ for some constant $c_{S}$.
Then for any $\delta\in(0,1)$, if $p = \Omega(\frac{Md\log(d/\delta)}{\gamma})$, the following inequality holds for all $t\geq 0$ with probability at least $1-\delta$:
\begin{equation}
\label{eq:reg_error_bd}
\|e(t+1)\|\leq \big(1-\frac{\eta\gamma}{2}-\eta\lambda(t)\big)\|e(t)\| + \lambda(t)\|y\|.
\end{equation}
\end{theorem}

We remark that in the linear case, the kernel matrix $G$ reduces to the form $X W\transpose W X\transpose$ and its expectation $\overline{G}$ at initialization also reduces to $X X\transpose$. Thus, Assumption \ref{assump:G} holds if $XX\transpose$ is positive definite, which is equivalent to the $x_i$'s being linearly independent. The result of Theorem \ref{thm:nonlinear_conv_reg} can not be directly applied to the linear case since we assume that $\psi$ is bounded, which is true for sigmoid or $\tanh$ but not for the identity map. This results in a slightly different order for $S_\lambda$ and an improved order for $p$.

Our results on alignment also rely on an isometric condition on $X$, which requires the minimum and the maximum eigenvalues of $XX\transpose$ to be sufficiently close (\cf Definition \ref{def:isom}). On the other hand, this condition is relatively mild and can be satisfied when $X$ has random Gaussian entries with a gentle dimensional constraint, as demonstrated by Proposition \ref{prop:isom}. Finally, we show in Theorem \ref{thm:lin_align} that under a simple regularization strategy where a constant regularization is adopted until a cutoff time $T$, regularized feedback alignment achieves alignment if $X$ satisfies the isometric condition.
{}
\begin{definition}[$(\gamma, \eps)$-Isometry]
\label{def:isom}
Given positive constants $\gamma$ and $\eps$, we say $X$ is $(\gamma, \eps)$-isometric if
$\lambda_{\min}(XX\transpose)\geq\gamma$ and $\lambda_{\max}(XX\transpose)\leq(1+\eps)\gamma$.
\end{definition}

\begin{proposition}
\label{prop:isom}
Assume $X\in\R^{n\times d}$ has independent entries drawn from $N(0,1/d)$. For any $\eps \in (0,1/2)$ and $\delta \in (0,1)$, if $d=\Omega(\frac{1}{\eps}\log\frac{n}{\delta}+\frac{n}{\eps}\log \frac{1}{\eps})$, then $X$ is $(1-\eps, 4\eps)$-isometric with probability $1-\delta$.
\end{proposition}

\begin{theorem}
\label{thm:lin_align}
Assume all conditions from Theorem \ref{thm:lin_conv} hold and $X$ is $(\gamma, \eps)$-isometric with a small constant $\eps$. Let the regularization weights satisfy
\begin{align*}
\lambda(t) =
\begin{cases}
    \lambda, \quad t\leq T,\\
    0, \quad t > T,
\end{cases}
\end{align*}
with $\lambda=L\gamma$ and $T = \lfloor S_\lambda/\lambda\rfloor$ for some large constant $L$. Then for any $\delta\in(0,1)$, if $p = \Omega(d\log(d/\delta))$, with probability at least $1-\delta$, regularized feedback alignment  achieves alignment. Specifically, there exist a positive constant $c=c_\delta$ and time $T_c$, such that $\cos\angle(b, \beta(t))\geq c$ for all $t>T_c$.
\end{theorem}

We defer the proofs of Proposition \ref{prop:isom}, Theorem \ref{thm:lin_conv} and Theorem \ref{thm:lin_align} to \cref{sec:appendix-alignment}. In fact, we prove Theorem \ref{thm:lin_align} by directly computing $\beta(t)$ and the cosine of the angle. Although $b$ doesn't show up in the update of $\beta$, it can still propagate to $\beta$ through $W$. Since the size of the component of $b$ in $\beta(t)$ depends on the inner-product $\langle e(t), e(t')\rangle$ for all previous steps $t'\leq t$, the norm bound \eqref{eq:reg_error_bd} from Theorem \ref{thm:lin_conv} is insufficient; thus, a more careful analysis of $e(t)$ is required.

We should point out that the constant $c$ in the lower bound is independent of the sample size $n$, input dimension $d$, network width $p$ and learning rate $\eta$. We also remark that the cutoff schedule of $\lambda(t)$ is just chosen for simplicity. For other schedules such as inverse-squared decay or exponential decay, one could also obtain the same alignment result as long as the summation of $\lambda(t)$ is less than $S_\lambda$.

\paragraph{Large sample scenario.} In Theorems \ref{thm:lin_conv} and \ref{thm:lin_align}, we consider the case where the sample size $n$ is less than the input dimension $d$, so that positive definiteness of $XX\transpose$ can be established. However, both results still hold for $n>d$. In fact, the squared error loss $\calL$ can be written as
\begin{equation*}
\sum_{i=1}^n\big(f(x_i)-y\big)^2 = \big\|\frac{1}{\sqrt{p}}XW\transpose\beta-y\big\|^2 = \big\|\frac{1}{\sqrt{p}}XW\transpose\beta-\bar{y}\big\|^2+ \|\bar{y}-y\|^2,
\end{equation*}
where $\bar{y}$ denotes the projection of $y$ onto the column space of $X$. Without loss of generality, we assume $y=\bar{y}$. As a result, $y$ and the columns of $X$ are all in the same $d$-dimensional subspace of $\Rn$ and $XX\transpose$ is positive definite on this subspace, as long as $X$ has full column rank. Consequently, we can either work on this subspace of $\Rn$ or project all the vectors onto $\Rd$, and the isometric condition is revised to only consider the $d$ nonzero eigenvalues of $XX\transpose$.

%% file: simulations.tex
\section{Simulations}

Our experiments apply the feedback alignment algorithm to two-layer networks, using a range of networks with different widths and activations. The numerical results suggest that regularization is essential in achieving alignment, in both regression and classification tasks, for linear and nonlinear models. We implement the feedback alignment procedure in PyTorch as an extension of the autograd module for backpropagation, and the training is done on V100 GPUs from internal clusters.

\paragraph{Feedback alignment on synthetic data.}

We first train two-layer networks on synthetic data, where each network $f$ shares the architecture shown in \eqref{eqn:nonlinear-network} and the data are generated by another network $f_0$ that has the same architecture but with random Gaussian weights. We present the experiments for both linear and nonlinear networks, where the activation functions are chosen to be Rectified Linear Unit (ReLU) and hyperbolic tangent (Tanh) for nonlinear case. We set training sample sample size to $n=50$ and the input dimension $d=150$, but vary the hidden layer width $p = 100\times 2^k$ with $k\in[7]$. During training, we take step size $\eta = 10^{-4}$ for linear networks and $\eta = 10^{-3},10^{-2}$ for ReLU and Tanh networks, respectively.

\begin{figure}[ht]
\centering
\begin{subfigure}[b]{.33\textwidth}
  \centering
  \includegraphics[width=\linewidth]{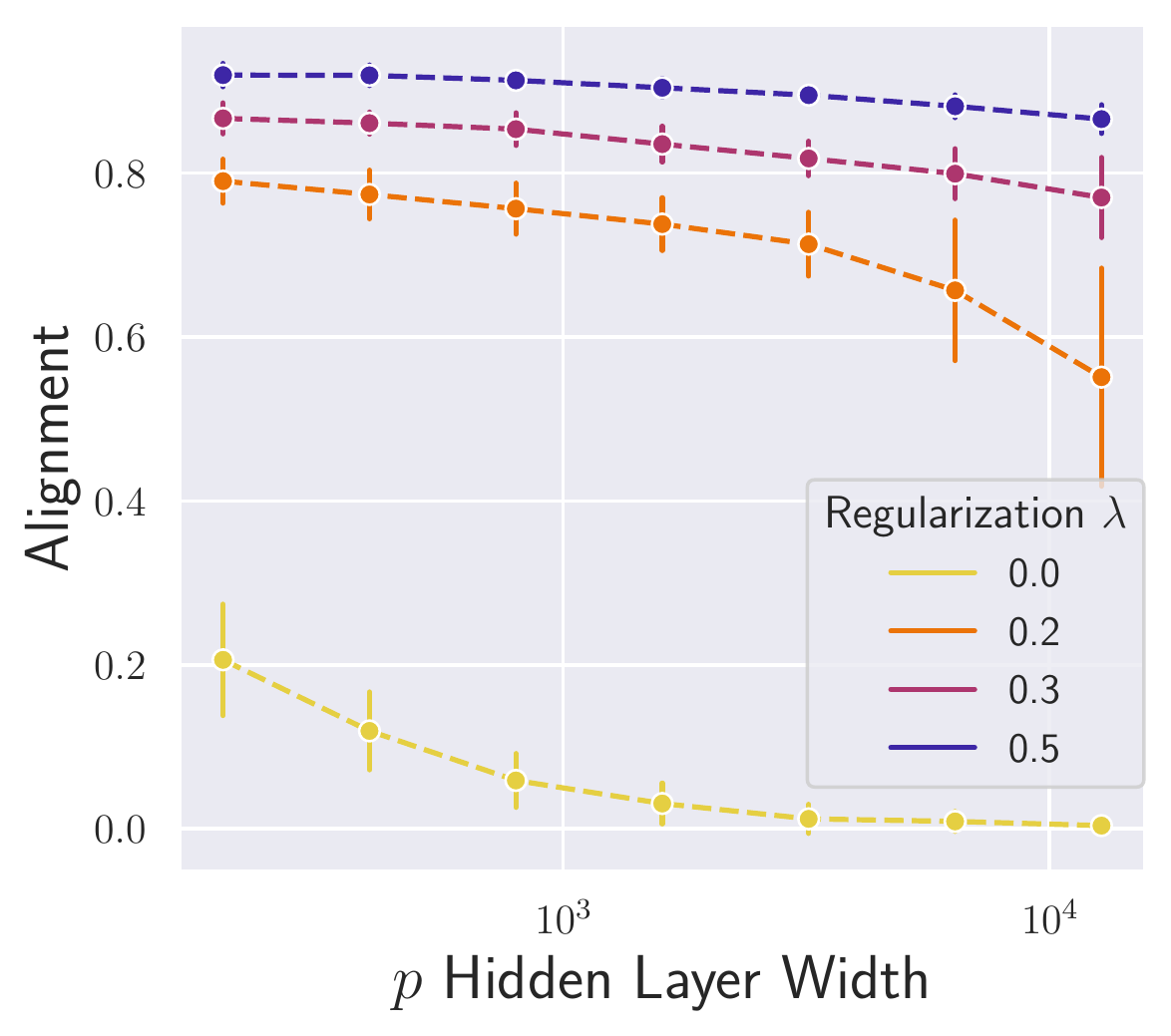}
  \caption{Alignment on linear network.}
  \label{fig:align_lr_non_autograd_l2}
\end{subfigure}\hfill
\begin{subfigure}[b]{.33\textwidth}
  \centering
  \includegraphics[width=\linewidth]{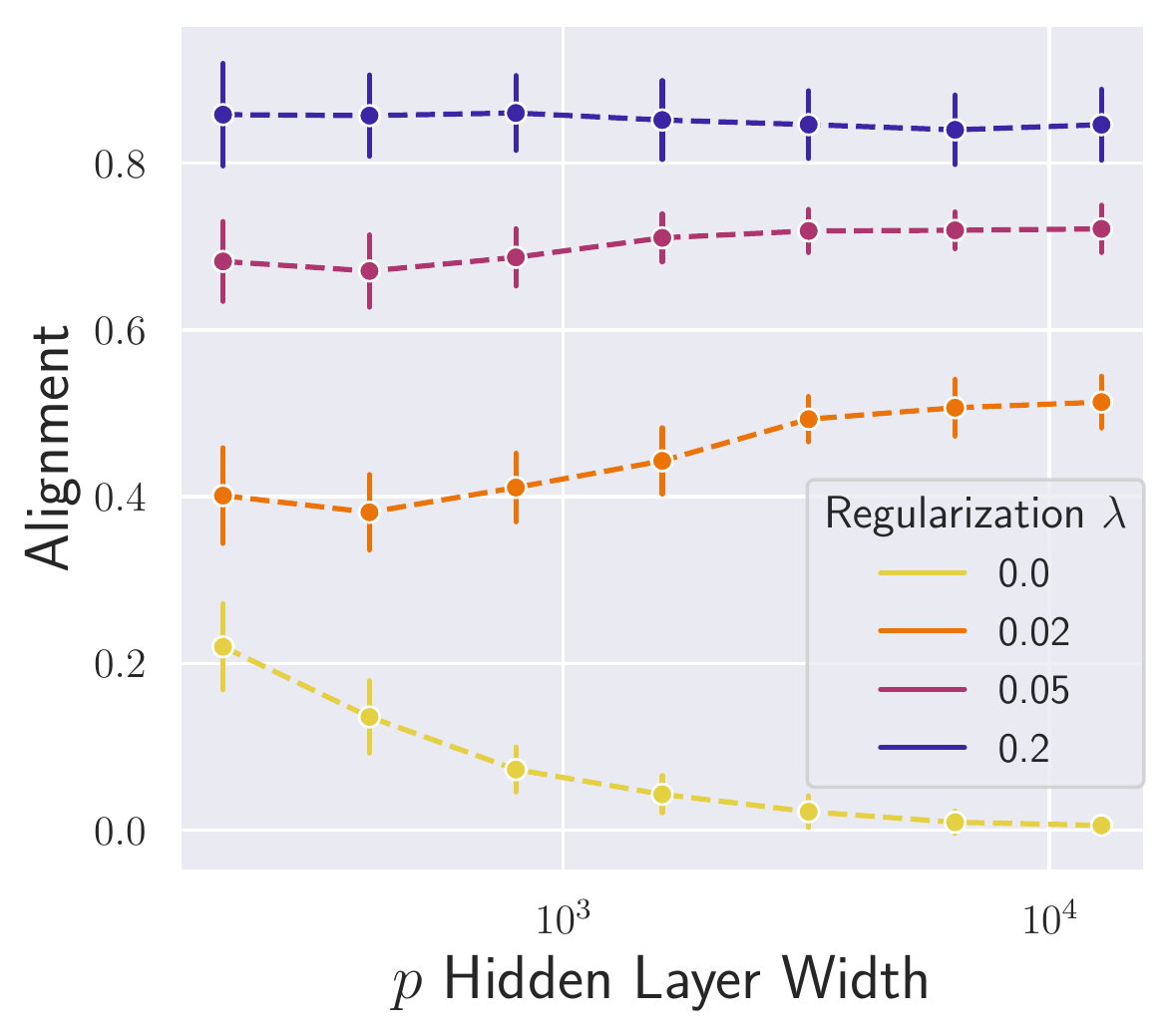}
  \caption{Alignment on ReLU network.}
  \label{fig:align_nn_relu_autograd_l2}
\end{subfigure}\hfill
\begin{subfigure}[b]{.33\textwidth}
  \centering
  \includegraphics[width=\linewidth]{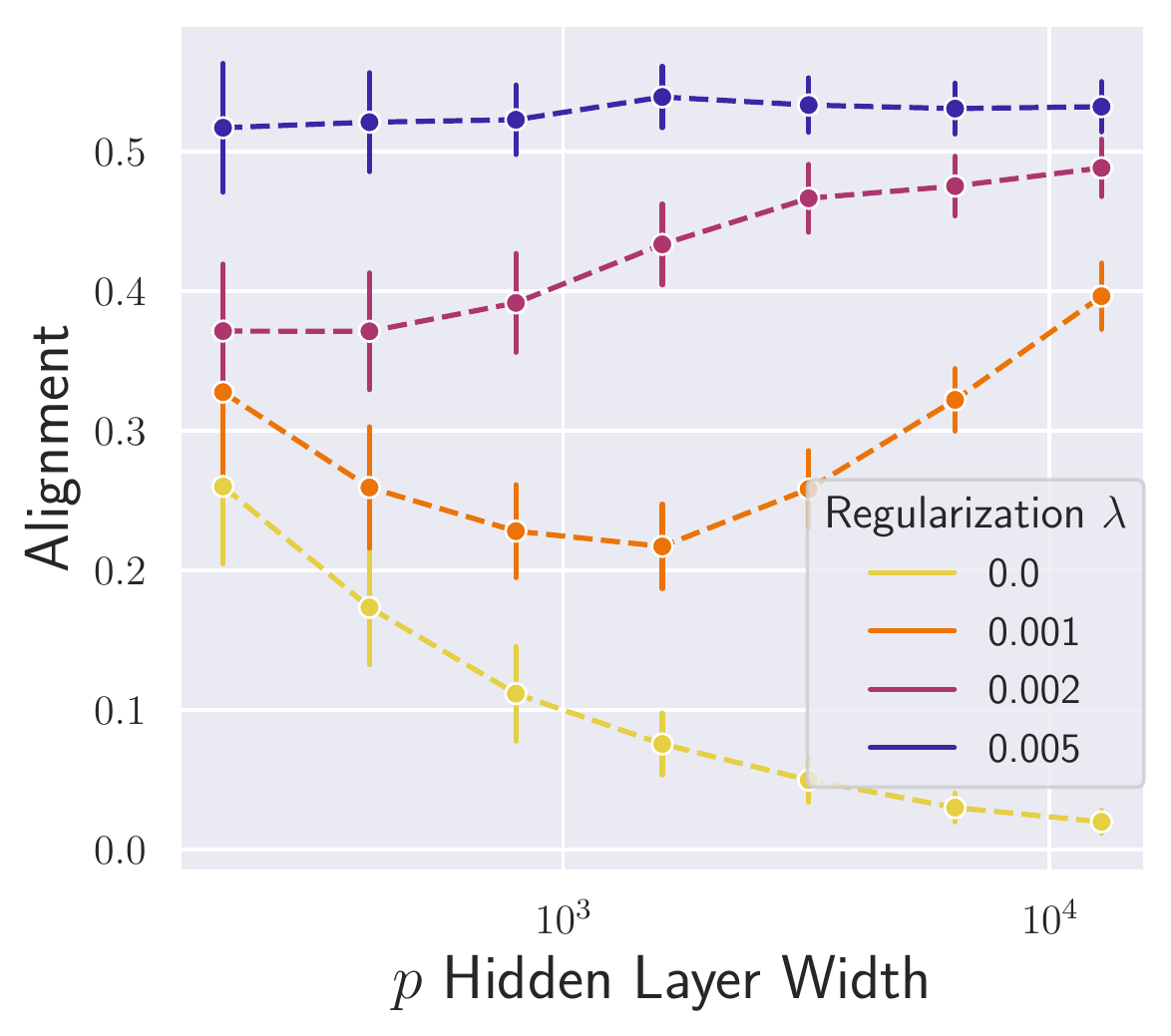}
  \caption{Alignment on Tanh network.}
  \label{fig:align_nn_tanh_autograd_l2}
\end{subfigure}
\medskip
\begin{subfigure}[b]{.33\textwidth}
  \centering
  \includegraphics[width=\linewidth]{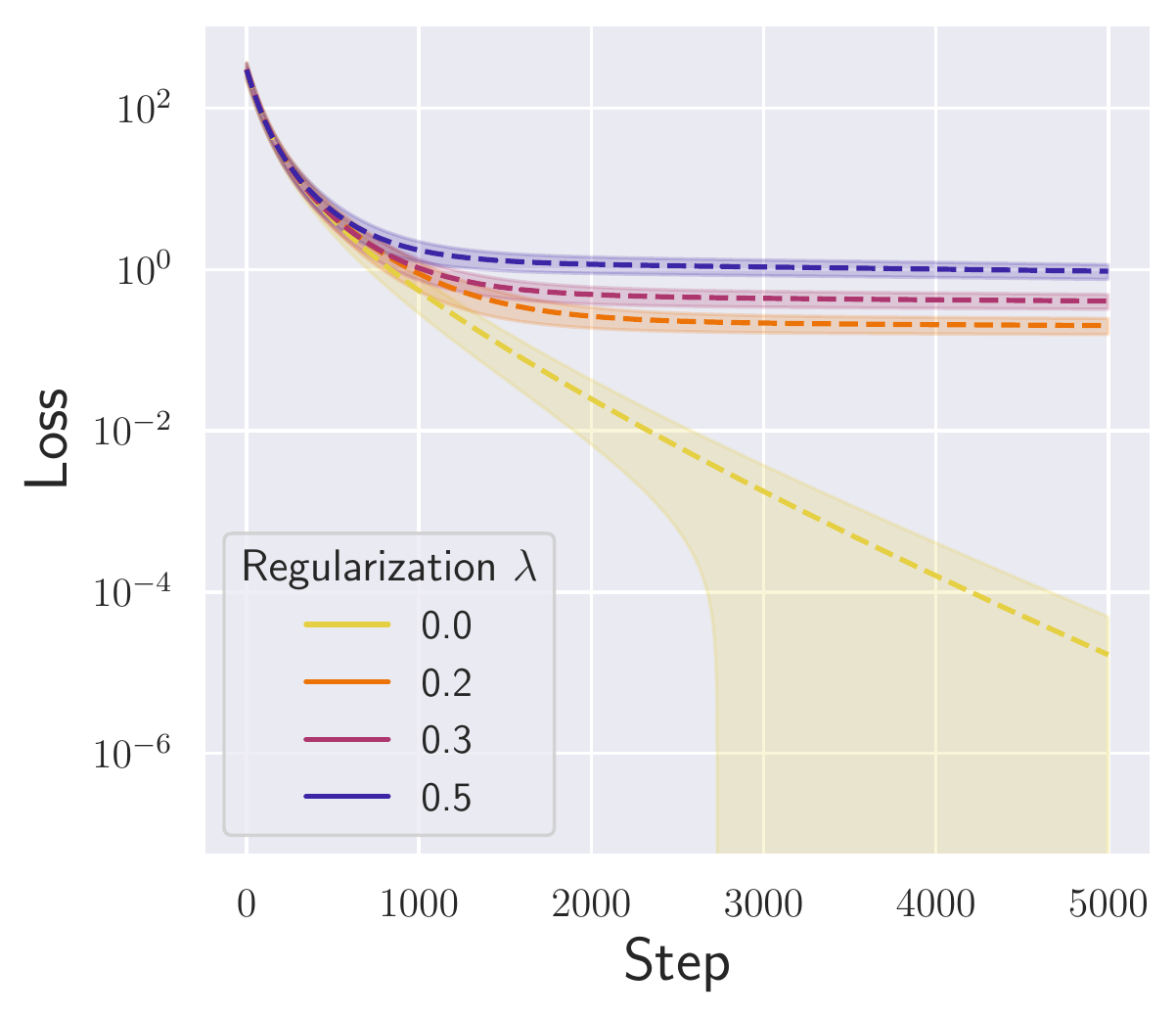}
  \caption{Loss on linear network.}
  \label{fig:loss_lr_non_autograd_l2}
\end{subfigure}\hfill
\begin{subfigure}[b]{.33\textwidth}
  \centering
  \includegraphics[width=\linewidth]{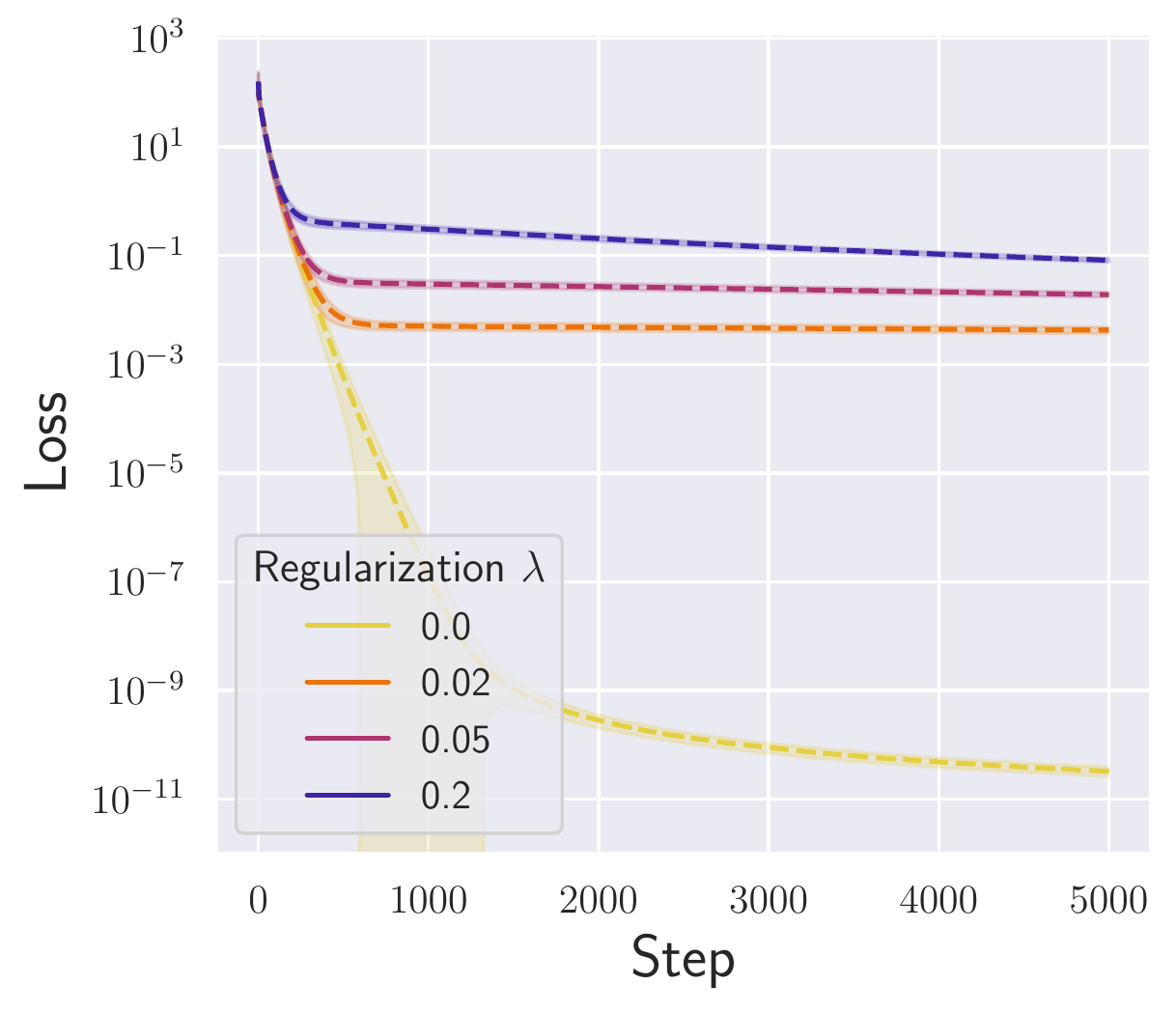}
  \caption{Loss on ReLU network.}
  \label{fig:loss_nn_relu_autograd_l2}
\end{subfigure}\hfill
\begin{subfigure}[b]{.33\textwidth}
  \centering
  \includegraphics[width=\linewidth]{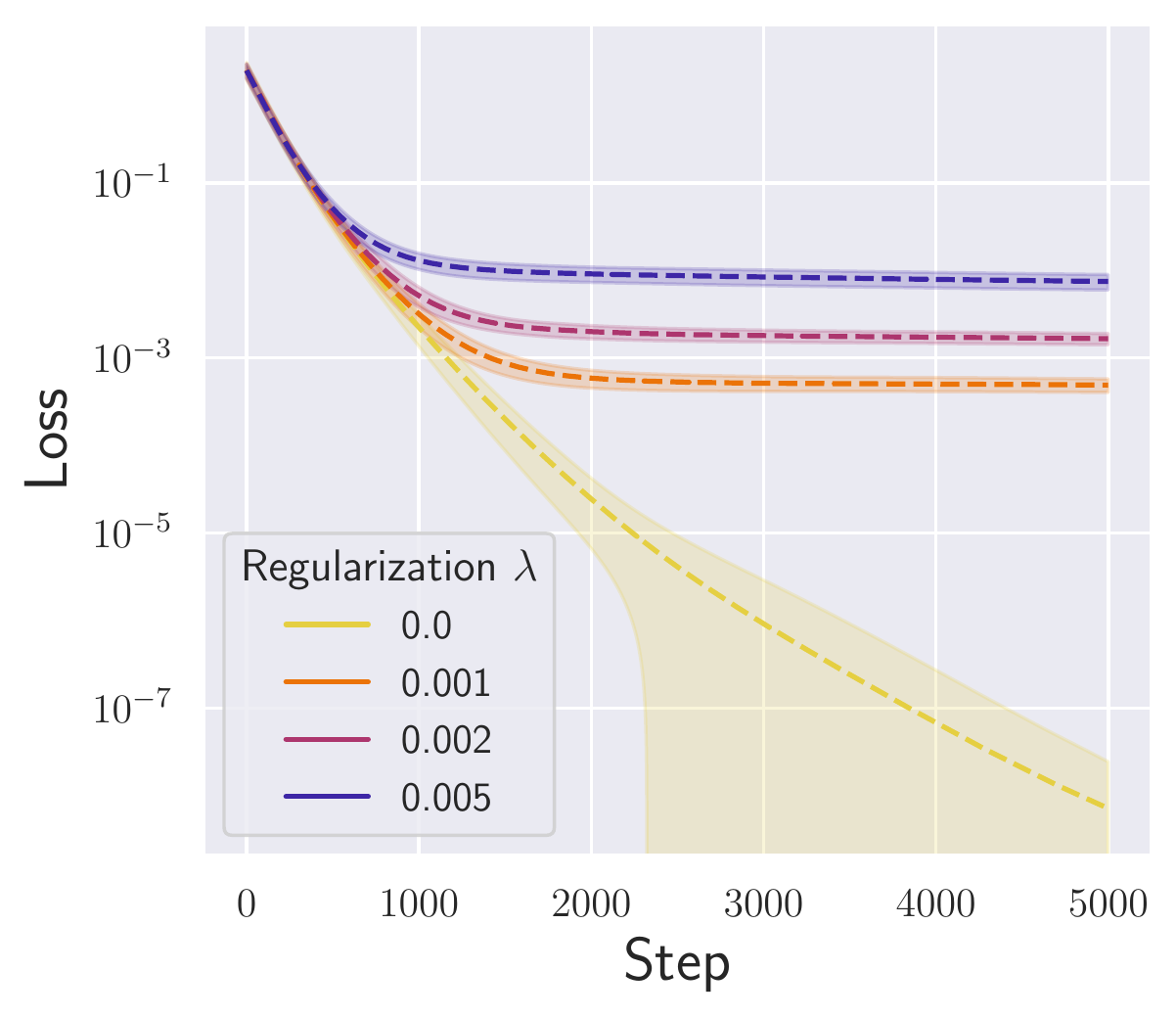}
  \caption{Loss on Tanh network.}
  \label{fig:loss_nn_tanh_autograd_l2}
\end{subfigure}
\caption{Comparisons of alignment and convergence for the feedback alignment algorithm with different levels of $\ell_2$ regularization. In \cref{fig:align_lr_non_autograd_l2,fig:align_nn_relu_autograd_l2,fig:align_nn_tanh_autograd_l2}, the data points represent the mean value computed across simulations, and the error bars mark the standard deviation out of $50$ independent runs. In \cref{fig:loss_lr_non_autograd_l2,fig:loss_nn_relu_autograd_l2,fig:loss_nn_tanh_autograd_l2}, we show the trajectories of the training loss for networks with $p = 3200$, with the shaded areas indicating the standard deviation over $50$ independent runs. The $x$-axes on the first row and the $y$-axes on the second row are presented using a logarithmic scale.}
\label{fig:synthetic-l2}
\end{figure}

In \cref{fig:align_lr_non_autograd_l2,fig:align_nn_relu_autograd_l2,fig:align_nn_tanh_autograd_l2}, we show how alignment depends on regularization and the degree of overparameterization as measured by the hidden layer width $p$. Alignment is measured by the cosine of the angle between the forward weights $\beta$ and backward weights $b$. We train the networks until the loss function converges; this procedure is repeated $50$ times for each $p$ and $\lambda$. For all three types of networks, as $p$ increases, alignment vanishes if there is no regularization, and grows with the level of regularization $\lambda$ for the same network. We complement the alignment plots with the corresponding loss curves, where the training loss converges slower with larger regularization. These numerical results are consistent with our theoretical statements. Due to the regularization, the loss converges to a positive number that is of the same order as $\lambda$.

We remark that using dropout as a form of regularization can also help the alignment between forward and backward weights \citep{wager2013dropout}. However, our numerical results suggest that dropout regularization fails to keep the alignment away from zero for networks with large hidden layer width. No theoretical result is available that explains the underlying mechanism.

\paragraph{Feedback alignment on the MNIST dataset.}

The \texttt{MNIST} dataset is available under the Creative Commons Attribution-Share Alike 3.0 license \citep{deng2012mnist}. It consists of 60,000 training images and 10,000 test images of dimension $28$ by $28$. We reshape them into vectors of length $d = 784$ and normalize them by their mean and standard deviation. The network structure is $784$-$1000$-$10$ with ReLU activation at the hidden layer and with softmax normalization at output layer. During training, we choose the batch size to be $600$ and the step size $\eta = 10^{-2}$. The training procedure uses $300$ epochs in total. We repeat the training 10 times for each choice of $\lambda$.

\cref{fig:mnist} shows the performance of feedback alignment with regularization $\lambda = 0, 0.1, 0.3$. Since the output of the network is not one-dimensional but 10-dimensional, the alignment is now measured by $\cos \angle(\dbp(h),\dfa(h))$, where $\dbp(h)$ is the error signal propagated to the hidden neurons $h$ through forward weights $\beta$, and $\dfa(h)$ the error weighted by the random backward weights $b$. We observe that both alignment and convergence are improved by adding regularization to the training, and increasing the regularization level $\lambda$ can further facilitate alignment, with a small gain in test accuracy.

\begin{figure}[t]
  \centering
  \includegraphics[width=\textwidth]{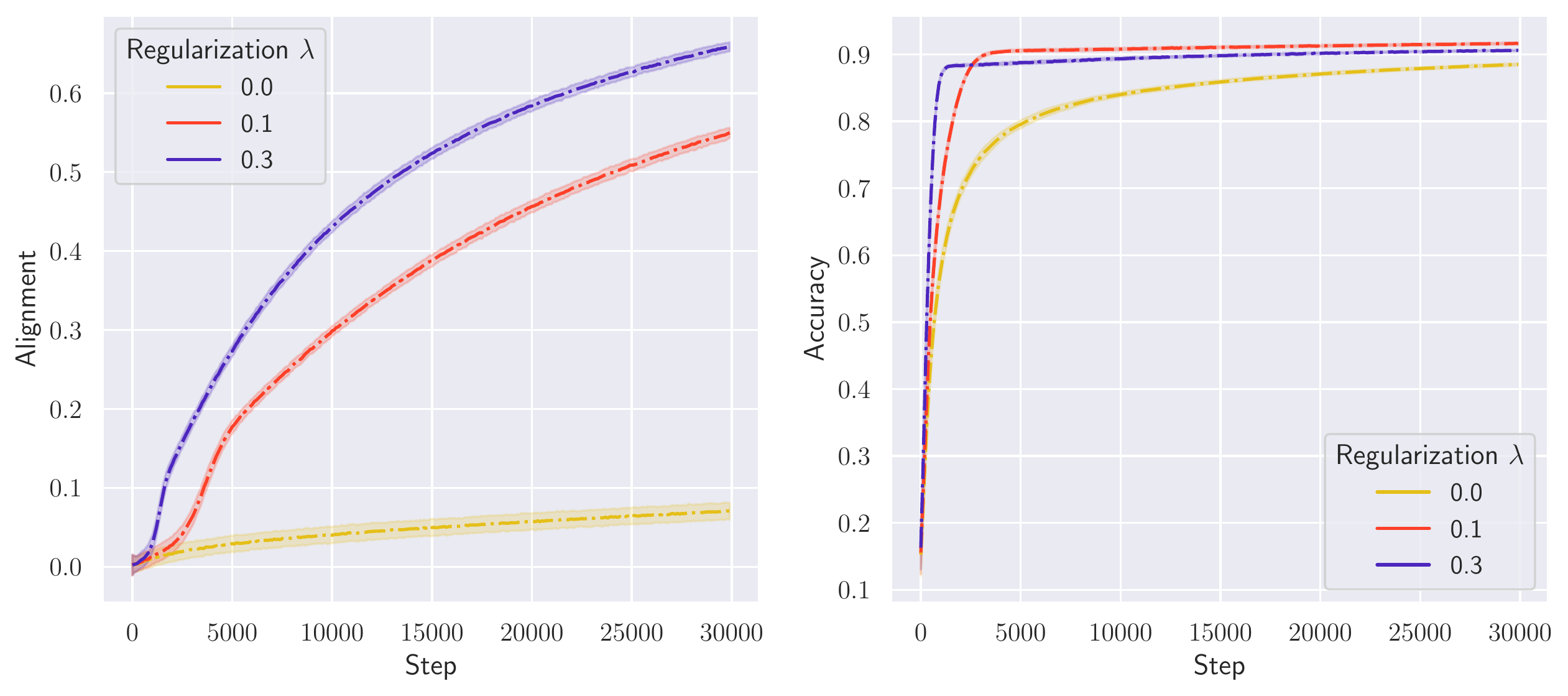}
  \caption{Comparisons on alignment and accuracy for feedback alignment algorithm with $\lambda=0,0.1,0.3$. The left figure shows alignment defined by $\cos \angle(\dbp(h),\dfa(h))$, and right figure shows the accuracy on the test set. The dashed lines and corresponding shaded areas represent the means and the standard deviations over $10$ runs with random initialization.}
  \label{fig:mnist}
\end{figure}

%% file: discussion.tex
\section{Discussion}

In this paper we analyzed the feedback alignment algorithm of
\cite{lillicrap2016random}, showing convergence of the algorithm. The convergence is subtle, as the algorithm does not directly minimize the target loss function; rather, the error is transferred to the hidden neurons through random weights that do not change during the course of learning.
The supplement to \cite{lillicrap2016random} presents interesting insights on the dynamics of the algorithm, such as how the feedback weights act as a pseudoinverse of the forward weights. After giving an analysis of convergence in the linear case, the authors state that
``a general proof must be radically different from those used to demonstrate convergence for backprop'' (Supplementary note 16), observing that the algorithm does not minimize any loss function. Our proof of convergence in the general nonlinear case leverages techniques from
the use of neural tangent kernel analysis in the over-parameterized setting, but requires more care because the kernel is not positive semi-definite at initialization. In particular, as a sum of two terms $G$ and $H$, the matrix $G$ is concentrated around its postive-definite mean, while $H$ is not generally postive-semidefinite. However, we show that the entries of both matrices remain close to their initial values, due to over-parameterization, and analyze the error term in a Taylor expansion to establish convergence.

In analyzing alignment, we found that regularization increases the alignment; without regularization, the alignment may not persist as the network becomes wider, as our simulations clearly show.
Our analysis in the linear case proceeds by establishing a recurrence of the form
$$\beta(t) = (1-\eta \lambda)^{t-1} \beta(0) + \frac{\eta}{\sqrt{p}} W(0) X\transpose \alpha_1(t-1) +
\left(\frac{\eta}{\sqrt{p}} \right) b \alpha_2(t-1)$$
and controlling $\alpha_1$ while showing that $\alpha_2$ remains sufficiently large; the
regularization kills off the first term.
Although we see no obstacle, in principle, to carrying out this proof strategy in the nonlinear
case, the calculations are more complex. While convergence requires analysis of the norm of the error, alignment requires understanding the direction of the error. But our simulations suggest this result will go through.

In terms of future research, a technical direction is to extend our results to multilayer networks. It would be interesting to explore local methods to update the backward weights $b$, rather than fixing them, perhaps using a Hebbian update rule in combination with the forward weights $W$. More generally, it is important to study other biologically plausible learning rules that can be implemented in deep learning frameworks at scale and without loss of performance.  The results presented here offer support for this as a fruitful line of research.

%% file: acknowledge.tex

\section*{Acknowledgments}

Research supported in part by NSF grant CCF-1839308.

%% file: proof_conv.tex

\section{Convergence on Two-Layer Nonlinear Networks}\label{sec:appendix-convergence}
We consider the family of neural networks
\begin{equation}
f(x) = \frac{1}{\sqrt p}\sum_{r=1}^p \beta_r \psi(w_r\transpose x) = \frac{1}{\sqrt p}\beta\transpose \psi(Wx)
\end{equation}
where $\beta \in \R^p$, $W = (w_1,...,w_p)\transpose\in\R^{p\times d}$, and $\psi$ is an activation function. Given data, the loss function is
\begin{equation}
\Loss(W, \beta) = \frac{1}{2}\sum_{i=1}^n(f(x_i)-y_i)^2 = \frac{1}{2}\sum_{i=1}^n\Big(\frac{1}{\sqrt p}\beta\transpose \psi(Wx_i)-y\Big)^2.
\end{equation}
The feedback alignment algorithm has updates
\begin{equation}
\label{eq:updates}
\begin{aligned}
    W(t+1) &= W(t) - \eta\frac{1}{\sqrt p}\sum_{i=1}^nD_i(t)bx_i\transpose e_i(t) \\
    \beta(t+1) &= \beta(t)- \eta \frac{1}{\sqrt p}\sum_{i=1}^n \psi(W(t)x_i)e_i(t)
\end{aligned}
\end{equation}
where $D_i(t) = \text{diag}(\psi'(W(t)x_i))$ and $e_i(t) = \frac{1}{\sqrt p}\beta(t)\transpose \psi(W(t)x_i)-y_i$. To help make the proof more readable, we use $c$, $C$ to denote the global constants whose values may vary from line to line.

\subsection{Concentration Results}

\begin{lemma}[Lemma A.7 in \citealp{gao2020model}]\label{lma:G}
Assume $x_1,...,x_n \overset{\iid}{\sim} \calN(0,I_d/d)$. We define  matrix $\widetilde{G}\in \R^{n\times n}$ with entries
\begin{equation*}
    \widetilde{G}_{i,j} = |\E\psi'(Z)|^2 \frac{x_i\transpose x_j}{\|x_i\|\|x_j\|} + (\E|\psi(Z)|^2-|\E\psi'(Z)|^2)\I\{i=j\}
\end{equation*}
where $Z\sim \calN(0,1)$. If $d = \Omega(\log n)$, then with high probability, we have
\begin{equation*}
\|\overline{G}-\widetilde{G}\|^2 \lesssim \frac{\log n}{d} + \frac{n^2}{d^2}.
\end{equation*}
\end{lemma}

\begin{proof}[Proof of Proposition \ref{prop:positive-definiteness}]
If $\psi$ is sigmoid or tanh, for a standard Gaussian random variable $Z$, we have
\begin{equation*}
    \gamma\defeq \frac{1}{2}(\E|\psi(Z)|^2-|\E\psi'(Z)|^2) >0.
\end{equation*}
From Lemma \ref{lma:G}, we know that with high probability $\lambda_{\min}(\overline{G}) \geq \lambda_{\min}(\widetilde{G})-\|\overline{G}-\widetilde{G}\|\geq 2\gamma - C(\sqrt{\frac{\log n}{d}} + \frac{n}{d}) \geq \gamma$.
\end{proof}

\begin{lemma}
\label{lma:inqs}
Assume $W(0)$, $\beta(0)$ and $b$ have \iid standard Gaussian entries. Given $\delta\in(0,1)$, if $p=\Omega(n/\delta)$, then with probability $1-\delta$
\begin{equation}\label{eq:sumb_bd}
    \frac{1}{p}\sum_{r=1}^p|b_r| \leq c,
\end{equation}
\begin{equation}\label{eq:sumbbet_bd}
    \frac{1}{p}\sum_{r=1}^p|b_r\beta_r(0)| \leq c,
\end{equation}
\begin{equation}\label{eq:e0_bd}
    \|e(0)\| \leq c\sqrt n,
\end{equation}
\begin{equation}\label{eq:maxb_bd}
    \max_{r\in[p]}|b_r|\leq 2\sqrt{\log p}.
\end{equation}
\end{lemma}
\begin{proof}
We will show each inequality holds with probability at least $1-\frac{\delta}{4}$, then by a union bound, all of them hold with probability at least $1-\delta$. Since $\Var(\frac{1}{p}\sum_{r=1}^p|b_r|)\leq \frac{\Var(|b_0|)}{p}$, by Chebyshev's inequality, we have
\begin{equation*}
    \prob(\frac{1}{p}\sum_{r=1}^p|b_r| > \E(b_1) + 1) \leq \frac{\Var(|b_1|)}{p} \leq \delta/4
\end{equation*}
if $p\geq 4\Var(|b_1|)/\delta$, which gives \eqref{eq:sumb_bd}. The proof for \eqref{eq:sumbbet_bd} is similar since $\Var(\frac{1}{p}\sum_{r=1}^p|b_r\beta_r(0)|)=O(1/p)$. To prove \eqref{eq:e0_bd}, since $|y_i|$ and $\|x_i\|$ are bounded, it suffices to show $|u_i(0)|\leq c$ for all $i\in [n]$. Actually, by independence, we have
\begin{equation*}
    \Var(u_i(0)) = \Var \Big(\frac{1}{p}\sum_{r=1}^p \beta_r(0)\psi(w_r(0)\transpose x_i)\Big) = \frac{1}{p}\Var\Big(\beta_1(0)\psi(w_1(0)\transpose x_i)\Big) = O(1/p).
\end{equation*}
By Chebyshev's inequality, we have for each $i\in [n]$
\begin{equation*}
    \prob(|u_i(0)|> c) \leq \frac{\Var(u_i(0))}{c^2} \leq \frac{\delta}{4n}
\end{equation*}
where we require $p=\Omega(n/\delta)$. With a union bound argument, we can show \eqref{eq:e0_bd}. Finally, \eqref{eq:maxb_bd} followed from standard Gaussian tail bounds and union bound argument, yielding
\begin{equation*}
    \prob(\max_{r\in[p]}|b_r| > 2\sqrt{\log p}) \leq \sum_{r\in [p]}\prob(|b_r| > 2\sqrt{\log p}) \leq 2pe^{-2\log p} = \frac{2}{p} \leq \frac{\delta}{4}.
\end{equation*}
\end{proof}

\begin{lemma}
\label{lma:GH}
Under the conditions of Theorem \ref{thm:nonliner_conv}, we define matrices $G(0),H(0)\in\R^{n\times n}$ with entries
\begin{equation}
\label{eq:def_G0}
G_{ij}(0) = \frac{1}{p}\psi(W(0)x_i)\transpose \psi(W(0)x_j) = \frac{1}{p}\sum_{r=1}^p\psi(w_r(0)\transpose x_i)\psi(w_r(0)\transpose x_j)
\end{equation}
and
\begin{equation}
\label{eq:def_H0}
H_{ij}(0) = \frac{x_i\transpose x_j}{p}\beta(0)\transpose D_i(0)D_j(0)b = \frac{1}{p}\sum_{r=1}^p\beta_r(0)b_r\psi'(w_r(0)\transpose x_i)\psi'(w_r(0)\transpose x_j).
\end{equation}
For any $\delta \in (0,1)$, if $p=\Omega(\frac{n^2}{\delta\gamma^2})$, then with probability at least $1-\delta$, we have $\lambda_{\min}(G(0))\geq \frac{3}{4}\gamma$ and $\|H(0)\|\leq \frac{\gamma}{4}$.
\end{lemma}
\begin{proof}
By independence and boundedness of $\psi$ and $\psi'$, we have $\Var(G_{ij}(0)) = O(1/p)$ and $\Var(H_{ij}(0)) = O(1/p)$. Since $\E(G(0))=\overline{G}$, we have
\begin{equation*}
\E\|G(0)-\overline{G}\|^2 \leq \E\|G(0)-\overline{G}\|^2_F = O(\frac{n^2}{p}).
\end{equation*}
By Markov's inequality, when $p=\Omega(\frac{n^2}{\delta\gamma^2})$
\begin{equation*}
    \prob(\|G(0)-\overline{G}\|>\frac{\gamma}{4})\leq O(\frac{n^2}{p\gamma^2})\leq\frac{\delta}{2}.
\end{equation*}
Similarly we have $\prob(\|H(0)\|>\frac{\gamma}{4})\leq\frac{\delta}{2}$, since $\E(H(0))=0$. Then with probability at least $1-\delta$, $\lambda_{\min}(G(0)) \geq \lambda_{\min}(\overline{G}) -\gamma/4 \geq \frac{3}{4}\gamma$, and $\|H(0)\|\leq \gamma/4$.
\end{proof}

\subsection{Proof of Theorem \ref{thm:nonliner_conv}}

\begin{lemma}
\label{lma:weights}
Assume all the inequalities from Lemma \ref{lma:inqs} hold. Under the conditions of Theorem \ref{thm:nonliner_conv}, if the error bound \eqref{eq:conv} holds for all $t=1,2,...,t'-1$, then the bounds \eqref{eq:weights} hold for all $t\leq t'$.
\end{lemma}

\begin{proof}
From the feedback alignment updates \eqref{eq:updates}, we have for all $t\leq T$
\begin{equation*}
\begin{aligned}
    |\beta_r(t)-\beta_r(0)| &\leq \frac{\eta}{\sqrt p}\sum_{s=0}^{t-1}\sum_{i=1}^n |\psi(w_r(t)x_i)e_i(t)| \\
    &\leq c\frac{\eta}{\sqrt p}\sum_{s=0}^{t-1}\sum_{i=1}^n |e_i(t)| \\
    &\leq c\frac{\eta\sqrt n}{\sqrt p}\sum_{s=0}^{t-1} \|e(t)\| \\
    &\leq c\frac{\eta\sqrt n}{\sqrt p}\sum_{s=0}^{t-1}  (1-\frac{\gamma\eta}{4})^t\|e(0)\|  \\
    &\leq c\frac{\sqrt n}{\gamma\sqrt p}\|e(0)\| \\
    &\leq c\frac{n}{\gamma\sqrt p}
\end{aligned}
\end{equation*}
where we use the fact that $\psi$ is bounded and \eqref{eq:e0_bd}. We also have
\begin{equation*}
\begin{aligned}
    \|w_r(t)-w_r(0)\| &\leq \frac{\eta}{\sqrt p}\sum_{s=0}^{t-1}\sum_{i=1}^n \|\psi'(w_r(t)\transpose x_i)b_rx_i e_i(t)\| \\
    & \leq c\frac{\eta}{\sqrt p}\sum_{s=0}^{t-1}\sum_{i=1}^n |b_r| |e_i(t)| \\
    & \leq c|b_r|\frac{\eta\sqrt n}{\sqrt p}\sum_{s=0}^{t-1} \|e(t)\| \\
    & \leq c|b_r|\frac{\sqrt n}{\gamma\sqrt p}\|e(0)\| \\
    & \leq c\frac{n\sqrt{\log p}}{\gamma\sqrt p}
\end{aligned}
\end{equation*}
where we use that $\psi'$ is bounded, \eqref{eq:e0_bd} and \eqref{eq:maxb_bd}.
\end{proof}

\begin{lemma}
\label{lma:induction}
Assume all the inequalities from Lemma \ref{lma:inqs} hold. Under the conditions of Theorem \ref{thm:nonliner_conv}, if the bound for the weights difference \eqref{eq:weights} holds for all $t\leq t'$ and error bound \eqref{eq:conv} holds for all $t\leq t'-1$, then \eqref{eq:conv} holds for $t=t'$.
\end{lemma}
\begin{proof}
We start with analyzing the error $e(t)$ according to
\begin{equation*}
\begin{aligned}
    e_i(t+1)
    &= \frac{1}{\sqrt p}\beta(t+1)\transpose\psi(W(t+1)x_i) - y_i \\
    &= \frac{1}{\sqrt p}\beta(t+1)\transpose(\psi(W(t+1)x_i)-\psi(W(t)x_i))+ \frac{1}{\sqrt p}(\beta(t+1)-\beta(t))\transpose \psi(W(t)x_i) \\
    & \quad + \frac{1}{\sqrt p}\beta(t)\transpose \psi(W(t)x_i) - y_i \\
    &=e_i(t) - \frac{\eta}{p}\beta(t+1)\transpose D_i(t)\sum_{j=1}^nD_j(t)b x_j\transpose x_i e_j(t)  - \frac{\eta}{p}\sum_{j=1}^n\psi(W(t)x_j)\transpose\psi(W(t)x_i)e_j(t) \\
    & \quad + v_i(t) \\
    & = e_i(t)-\eta\sum_{j=1}^n\big(H_{ij}(t)+G_{ij}(t)\big)e_j(t) + v_i(t)
\end{aligned}
\end{equation*}
where
\begin{equation*}
\begin{aligned}
G_{ij}(t) &= \frac{1}{p}\psi(W(t)x_j)\transpose\psi(W(t)x_i) \\
H_{ij}(t) &= \frac{x_i\transpose x_j}{p}\beta(t+1)\transpose D_i(t)D_j(t)b
\end{aligned}
\end{equation*}
and $v_i(t)$ is the residual term from the Taylor expansion
\begin{equation*}
    v_i(t) = \frac{1}{2\sqrt p}\sum_{r=1}^p\beta_r(t+1)|(w_r(t+1)-w_r(t))\transpose x_i|^2\psi''(\xi_{ri}(t))
\end{equation*}
with $\xi_{ri}(t)$ between $w_r(t)\transpose x_i$ and $w_r(t+1)\transpose x_i$. We can also rewrite the above iteration in vector form as
\begin{equation}\label{eq:et_iter}
     e(t+1) = e(t) - \eta(G(t)+H(t))e(t) + v(t).
\end{equation}
Now for $t=t'-1$, we wish to show that both $G(t)$ and $H(t)$ are close to their initialization. Notice that
\begin{equation*}
\begin{aligned}
    |G_{ij}(t) - G_{ij}(0)|
    & = \frac{1}{p}\Big|\psi(W(t)x_j)\transpose\psi(W(t)x_i) -\psi(W(t)x_j)\transpose\psi(W(t)x_i)\Big| \\
    & \leq \frac{1}{p} \sum_{r=1}^p |\psi(w_r(t)\transpose x_j)||\psi(w_r(t)\transpose x_i)-\psi(w_r(0)\transpose x_i)| \\
    &\quad + \frac{1}{p} \sum_{r=1}^p |\psi(w_r(0)\transpose x_i)||\psi(w_r(t)\transpose x_j)-\psi(w_r(0)\transpose x_j)| \\
    &\leq c \frac{1}{p} \sum_{r=1}^p|w_r(t)\transpose x_i-w_r(0)\transpose x_i| + \frac{1}{p} \sum_{r=1}^p|w_r(t)\transpose x_j-w_r(0)\transpose x_j| \\
    &\leq c_0 \frac{n\sqrt{\log p}}{\gamma\sqrt p} (\|x_i\|+\|x_j\|)
\end{aligned}
\end{equation*}
where the second inequality is due to the boundedness of $\psi$ and $\psi'$, and the last inequality is by \eqref{eq:weights}. Then we have
\begin{equation}\label{eq:bound_Gt}
    \|G(t)-G(0)\| \leq \max_{j \in [n]}\sum_{i=1}^n|G_{ij}(t) - G_{ij}(0)| \leq c_0 \frac{n^2\sqrt{\log p}}{\gamma\sqrt p}.
\end{equation}
For matrix $H(t)$, we similarly have
\begin{equation*}
\begin{aligned}
    |H_{ij}(t)-H_{ij}(0)|
    &\leq \frac{|x_i\transpose x_j|}{p}\Big|\beta(t+1)\transpose D_i(t)D_j(t)b - \beta(0)\transpose D_i(0)D_j(0)b\Big| \\
    &\leq \frac{\|x_i\|\|x_j\|}{p}\sum_{r=1}^p \Big|b_r\beta_r(t+1)\psi'(w_r(t)\transpose x_i)\psi'(w_r(t)\transpose x_j)\\
    &\quad -b_r\beta_r(0)\psi'(w_r(0)\transpose x_i)\psi'(w_r(0)\transpose x_j)\Big| \\
    &\leq \frac{|\|x_i\|\|x_j\||}{p}\sum_{r=1}^p\Big(|b_r||\beta_r(t+1)-\beta_r(0)| |\psi'(w_r(t)\transpose x_i)\psi'(w_r(t)\transpose x_j)|\\
    & \quad +|b_r||\beta_r(0)| |\psi'(w_r(t)\transpose x_i)-\psi'(w_r(0)\transpose x_i)| |\psi'(w_r(t)\transpose x_j)|\\
    & \quad +|b_r||\beta_r(0)| |\psi'(w_r(0)\transpose x_i)| |\psi'(w_r(t)\transpose x_j)-\psi'(w_r(0)\transpose x_j)| \Big)\\
    &\leq c\frac{\|x_i\|\|x_j\|}{p}\sum_{r=1}^p\Big(|b_r|\frac{n}{\gamma\sqrt p} + |b_r||\beta_r(0)|\frac{n\sqrt{\log p}}{\gamma\sqrt p} (\|x_i\|+\|x_j\|)\Big) \\
    &\leq c_1 \frac{n}{\gamma\sqrt p} + c_2\frac{n\sqrt{\log p}}{\gamma\sqrt p}.
\end{aligned}
\end{equation*}
It follows that
\begin{equation}\label{eq:bound_Ht}
    \|H(t)-H(0)\| \leq \max_{j \in [n]}\sum_{i=1}^n|H_{ij}(t) - H_{ij}(0)| \leq c_1 \frac{n^2}{\gamma\sqrt p} + c_2\frac{n^2\sqrt{\log p}}{\gamma\sqrt p}.
\end{equation}
Next, we bound the residual term $v_i(t)$. Since $\psi''$ is bounded, we have
\begin{equation*}
\begin{aligned}
    |v_i(t)|
    &\leq c\frac{1}{\sqrt p}\sum_{r=1}^p|\beta_r(t+1)|\|w_r(t+1)-w_r(t)\|^2 \\
    &\leq c\frac{1}{\sqrt p}\frac{\eta^2}{p}\sum_{r=1}^p|\beta_r(t+1)|\Big(\sum_{i=1}^n\|\psi'(w_r(t)\transpose x_i)b_rx_ie_i(t)\|\Big)^2 \\
    &\leq c\frac{1}{\sqrt p}\frac{\eta^2}{p}\sum_{r=1}^p|\beta_r(t+1)||b_r|^2\Big(\sum_{i=1}^n|e_i(t)|\Big)^2 \\
    &\leq c\frac{\eta^2n}{\sqrt p}\|e(t)\|^2 \\
    &\leq c_3\frac{\eta^2n\sqrt n}{\sqrt p}\|e(t)\|.
\end{aligned}
\end{equation*}
This leads to the bound
\begin{equation}\label{eq:bound_vt}
    \|v(t)\| =\Big(\sum_{i=1}^n|v_i(t)|^2\Big)^{1/2} \leq c_3\frac{\eta^2n^2}{\sqrt p}\|e(t)\|.
\end{equation}
Combining \cref{eq:et_iter,eq:bound_Gt,eq:bound_Ht,eq:bound_vt}, we have
\begin{equation*}
\begin{aligned}
\|e(t+1)\|
&\leq \|I_n-\eta (G(t)+H(t))\|\|e(t)\|+\|v(t)\| \\
&\leq \Big(\|I_n-\eta G(0)\|+\eta\|G(t)-G(0)\|+\eta\|H(0)\| \\
&\quad +\eta\|H(t)-H(0)\|\Big)\|e(t)\| + \|v(t)\| \\
&\leq \Big( 1-\frac{3\eta\gamma}{4}+c_0\frac{\eta n^2\sqrt{\log p}}{\gamma\sqrt p}+\frac{\eta\gamma}{4}+c_1\frac{\eta n^2}{\gamma\sqrt p} + c_2\frac{\eta n^2\sqrt{\log p}}{\gamma\sqrt p}+c_3\frac{\eta^2n\sqrt n}{\sqrt p}\Big)\|e(t)\|  \\
&\leq(1-\frac{\eta\gamma}{4})\|e(t)\|
\end{aligned}
\end{equation*}
where we use Lemma \ref{lma:GH} and $p=\Omega(\frac{n^4\log p}{\gamma^4})$.
\end{proof}

\begin{proof}[Proof of Theorem \ref{thm:nonliner_conv}]
We prove the inequality \eqref{eq:conv} by induction. Suppose \eqref{eq:conv} and \eqref{eq:weights} hold for all $t=1,2,...,t'-1$, by Lemma \ref{lma:weights} and Lemma \ref{lma:induction} we know \eqref{eq:conv} and \eqref{eq:weights} hold for $t=t'$, which completes the proof.
\end{proof}

\subsection{Proof of Theorem \ref{thm:nonlinear_conv_reg}}

\begin{lemma}
\label{lma:weights_reg}
Assume all the inequalities from Lemma \ref{lma:inqs} hold. Under the conditions of Theorem \ref{thm:nonlinear_conv_reg}, if the error bound \eqref{eq:conv_reg} holds for all $t=1,2,...,t'-1$, then
\begin{equation}
\label{eq:weights_reg}
\begin{aligned}
    \|w_r(t)-w_r(0)\| &\leq c_1\frac{n\sqrt{\log p}}{\gamma\sqrt p}(1+\eta \tilde{S}_\lambda),\\
    |\beta_r(t)-\beta_r(0)| &\leq c_2\frac{n}{\gamma\sqrt p}(1+\eta \tilde{S}_\lambda)
\end{aligned}
\end{equation}
hold for all $t\leq t'$, where $c_1$, $c_2$ are constants.
\end{lemma}

\begin{proof}
For any $k\leq t'-1$, we apply \eqref{eq:conv_reg} repeatedly on the right hand side of itself to get
\begin{equation*}
    \|e(k)\|\leq \prod_{i=0}^{k-1}\Big(1-\frac{\eta\gamma}{4}-\eta\lambda(i)\Big)\|e(0)\| + \sum_{i=0}^{k-1}\eta\lambda(i)\prod_{i<j<k}\Big(1-\frac{\eta\gamma}{4}-\eta\lambda(j)\Big)\|y\|.
\end{equation*}
For $t\leq t'-1$, we take the sum over $k=0,..,t$ on both sides of above inequality
to obtain
\begin{equation*}
\begin{aligned}
    \sum_{k=0}^{t}\|e(k)\|
    &\leq \sum_{k=0}^{t}\prod_{i=0}^{k-1}\Big(1-\frac{\eta\gamma}{4}-\eta\lambda(i)\Big)\|e(0)\|  + \sum_{k=0}^{t}\sum_{i=0}^{k-1}\eta\lambda(i)\prod_{i<j<k}\Big(1-\frac{\eta\gamma}{4}-\eta\lambda(j)\Big)\|y\| \\
    &\leq \sum_{k=0}^{t}\Big(1-\frac{\eta\gamma}{4}\Big)^{k-1}\|e(0)\|+ \sum_{k=0}^{t}\sum_{i=0}^{k-1}\eta\lambda(i)\Big(1-\frac{\eta\gamma}{4}\Big)^{k-i-1}\|y\| \\
    &\leq \sum_{k=0}^{t}\Big(1-\frac{\eta\gamma}{4}\Big)^{k-1}\|e(0)\| + \eta\|y\|\sum_{k=0}^{t-1}\lambda(i)\sum_{k=i+1}^T\Big(1-\frac{\eta\gamma}{4}\Big)^{k-i-1} \\
    &\leq \frac{4}{\eta\gamma}\|e(0)\|+ \frac{4}{\gamma}\tilde{S}_\lambda\|y\| \\
    &\leq \frac{c\sqrt n}{\gamma}(\frac{1}{\eta}+\tilde{S}_\lambda)
\end{aligned}
\end{equation*}
where we use $\|e(0)\|=O(\sqrt n)$ and $\|y\|=O(\sqrt n)$. Then for all $t\leq t'$, we have
\begin{equation*}
\begin{aligned}
    |\beta_r(t)-\beta_r(0)| &\leq \frac{\eta}{\sqrt p}\sum_{s=0}^{t-1}\sum_{i=1}^n |\psi(w_r(t)x_i)e_i(t)| \\
    &\leq c\frac{\eta}{\sqrt p}\sum_{s=0}^{t-1}\sum_{i=1}^n |e_i(t)| \\
    &\leq c\frac{\eta\sqrt n}{\sqrt p}\sum_{s=0}^{t-1} \|e(t)\| \\
    &\leq c\frac{\eta\sqrt n}{\sqrt p}\frac{\sqrt n}{\gamma}(\frac{1}{\eta}+\tilde{S}_\lambda)  \\
    &\leq c\frac{n}{\gamma\sqrt p} (1+\eta \tilde{S}_\lambda)
\end{aligned}
\end{equation*}
where we use $\psi$ is bounded and \eqref{eq:e0_bd}. We also have
\begin{equation*}
\begin{aligned}
    \|w_r(t)-w_r(0)\| &\leq \frac{\eta}{\sqrt p}\sum_{s=0}^{t-1}\sum_{i=1}^n \|\psi'(w_r(t)\transpose x_i)b_rx_i e_i(t)\| \\
    & \leq c\frac{\eta}{\sqrt p}\sum_{s=0}^{t-1}\sum_{i=1}^n |b_r| |e_i(t)| \\
    & \leq c|b_r|\frac{\eta\sqrt n}{\sqrt p}\sum_{s=0}^{t-1} \|e(t)\| \\
    & \leq c|b_r|\frac{\eta\sqrt n}{\sqrt p}\frac{\sqrt n}{\gamma}(\frac{1}{\eta}+\tilde{S}_\lambda) \\
    & \leq c\frac{n\sqrt{\log p}}{\gamma\sqrt p} (1+\eta \tilde{S}_\lambda)
\end{aligned}
\end{equation*}
where we use the fact that $\psi'$ is bounded, \eqref{eq:e0_bd} and \eqref{eq:maxb_bd}.
\end{proof}

\begin{lemma}
\label{lma:induction_reg}
Assume all the inequalities from Lemma \ref{lma:inqs} hold. Under the conditions of Theorem \ref{thm:nonlinear_conv_reg}, if the bound for weights difference \eqref{eq:weights_reg} holds for all $t\leq t'$ and error bound \eqref{eq:conv_reg} holds for all $t\leq t'-1$, then \eqref{eq:conv_reg} holds for $t=t'$.
\end{lemma}
\begin{proof}
We start by analyzing the error $e(t)$ according to 
\begin{equation*}
\begin{aligned}
    e_i(t+1)
    &= \frac{1}{\sqrt p}\beta(t+1)\transpose\psi(W(t+1)x_i) - y_i \\
    &= \frac{1}{\sqrt p}\beta(t+1)\transpose(\psi(W(t+1)x_i)-\psi(W(t)x_i))+ \frac{1}{\sqrt p}(\beta(t+1)-(1-\eta\lambda(t))\beta(t))\transpose \psi(W(t)x_i) \\
    & \quad + (1-\eta\lambda(t))\Big(\frac{1}{\sqrt p}\beta(t)\transpose \psi(W(t)x_i) - y_i\Big) - \eta\lambda(t)y \\
    &=(1-\eta\lambda(t))e_i(t) - \frac{\eta}{p}\beta(t+1)\transpose D_i(t)\sum_{j=1}^nD_j(t)b x_j\transpose x_i e_j(t) \\
    & \quad - \frac{\eta}{p}\sum_{j=1}^n\psi(W(t)x_j)\transpose\psi(W(t)x_i)e_j(t) - \eta\lambda(t)y + v_i(t) \\
    & = (1-\eta\lambda(t))e_i(t)-\eta\sum_{j=1}^n\big(H_{ij}(t)+G_{ij}(t)\big)e_j(t) + v_i(t) - \eta\lambda(t)y
\end{aligned}
\end{equation*}
where
\begin{equation*}
\begin{aligned}
G_{ij}(t) &= \frac{1}{p}\psi(W(t)x_j)\transpose\psi(W(t)x_i) \\
H_{ij}(t) &= \frac{x_i\transpose x_j}{p}\beta(t+1)\transpose D_i(t)D_j(t)b
\end{aligned}
\end{equation*}
and $v_i(t)$ is the residual term from a Taylor expansion
\begin{equation*}
    v_i(t) = \frac{1}{2\sqrt p}\sum_{r=1}^p\beta_r(t+1)|(w_r(t+1)-w_r(t))\transpose x_i|^2\psi''(\xi_{ri}(t))
\end{equation*}
with $\xi_{ri}(t)$ between $w_r(t)\transpose x_i$ and $w_r(t+1)\transpose x_i$. We can also rewrite the above iteration in vector form as
\begin{equation}\label{eq:et_iter_reg}
     e(t+1) = (1-\lambda(t))e(t) - \eta(G(t)+H(t))e(t) + v(t) -\eta\lambda(t)y.
\end{equation}
Now for $t=t'-1$, we show that both $G(t)$ and $H(t)$ are close to their initialization. Using the argument in Lemma \ref{lma:induction}, we can obtain following bounds
\begin{equation}\label{eq:bound_Gt_reg}
    \|G(t)-G(0)\| \leq c_1 \frac{n^2\sqrt{\log p}}{\gamma\sqrt p}(1+\eta \tilde{S}_\lambda)
\end{equation}

\begin{equation}\label{eq:bound_Ht_reg}
    \|H(t)-H(0)\| \leq c_2\frac{n^2\sqrt{\log p}}{\gamma\sqrt p}(1+\eta \tilde{S}_\lambda)
\end{equation}
\begin{equation}\label{eq:bound_vt_reg}
    \|v(t)\| \leq c_3\frac{\eta^2n^2}{\sqrt p}\|e(t)\|.
\end{equation}
Combining \cref{eq:et_iter_reg,eq:bound_Gt_reg,eq:bound_Ht_reg,eq:bound_vt_reg}, we have
\begin{equation*}
\begin{aligned}
\|e(t+1)\|
&\leq \|(1-\eta\lambda(t))I_n-\eta (G(t)+H(t))\|\|e(t)\|+\|v(t)\| \\
&\leq \Big(\|(1-\eta\lambda(t))I_n-\eta G(0)\|+\eta\|G(t)-G(0)\|+\eta\|H(0)\| \\
&\quad +\eta\|H(t)-H(0)\|\Big)\|e(t)\| + \|v(t)\| \\
&\leq \Big( 1-\eta\lambda(t)-\frac{3\eta\gamma}{4}+(c_1+ c_2)\frac{\eta n^2\sqrt{\log p}}{\gamma\sqrt p}(1+\eta \tilde{S}_\lambda)+c_3\frac{\eta^2n\sqrt n}{\sqrt p}\Big)\|e(t)\|  \\
&\leq(1-\eta\lambda(t)-\frac{\eta\gamma}{4})\|e(t)\|
\end{aligned}
\end{equation*}
where we use Lemma \ref{lma:GH}, $p=\Omega(\frac{n^4\log p}{\gamma^4})$ and $\tilde{S}_\lambda = O(\frac{\gamma^2\sqrt{p}}{\eta n^2\sqrt{\log p}})$.
\end{proof}

\begin{proof}[Proof of Theorem \ref{thm:nonlinear_conv_reg}]
We prove the inequality \eqref{eq:conv_reg} by induction. Suppose \eqref{eq:conv_reg} holds for all $t=1,2,...,t'-1$. Then by Lemma \ref{lma:weights_reg} and Lemma \ref{lma:induction_reg} we know \eqref{eq:conv_reg} holds for $t=t'$, which completes the proof.
\end{proof}

%% file: proof_align.tex
\section{Alignment on Two-Layer Linear Networks}\label{sec:appendix-alignment}

Now we assume $\psi(u) = u$, so that $f$ is a linear network. The loss function with regularization at time $t$ is
\begin{equation}
\label{eq:reg_loss}
\Loss(t, W, \beta) = \frac{1}{2}\big\|\frac{1}{\sqrt p}XW\transpose\beta-y\big\|^2 + \frac{1}{2}\lambda(t)\|\beta\|^2.
\end{equation}
The regularized feedback alignment algorithm gives
\begin{equation}
\label{eq:reg_alg}
\begin{aligned}
    W(t+1) &= W(t) - \eta\frac{1}{\sqrt p}b e(t)\transpose X\\
    \beta(t+1) &= (1-\eta\lambda(t))\beta(t)- \frac{\eta}{\sqrt p} W(t)X\transpose e(t)
\end{aligned}
\end{equation}
where $e(t)=\frac{1}{\sqrt p}XW(t)\transpose \beta(t)-y$ is the error vector at time t.

\begin{lemma}
\label{lma:one_step_update}
Suppose the network is trained with the regularized feedback alignment algorithm \eqref{eq:reg_alg}. Then the prediction error $e(t)$ satisfies the recurrence
\begin{equation}
\label{eq:err_update}
\begin{aligned}
    e(t+1) = \bigg[(1-\eta\lambda(t))I_d - \frac{\eta}{p}XW(0)\transpose W(0)X\transpose -\eta\Big(J_1(t)+J_2(t)+J_3(t)\Big)\bigg]e(t) -\eta\lambda(t)y
\end{aligned}
\end{equation}
where
\begin{equation*}
\begin{aligned}
&J_1(t) = \frac{1}{p}b\transpose\beta(0)\prod_{i=0}^{t}(1-\eta\lambda(i))XX\transpose \\
&J_2(t) = -\frac{\eta}{p}\Big(\bar{v}\transpose X\transpose \hat{s}(t) XX\transpose+ XX\transpose s(t-1)\bar{v}\transpose X\transpose + X \bar{v}s(t-1)\transpose XX\transpose\Big) \\
&J_3(t) = \frac{\eta^2}{p^2}\|b\|^2\Big(\hat{S}(t)XX\transpose + XX\transpose s(t-1)s(t-1)\transpose XX\transpose)
\end{aligned}
\end{equation*}
and
\begin{equation*}
\begin{aligned}
&\bar{v} = \frac{1}{\sqrt p}W(0)\transpose b \\
&s(t) = \sum_{i=0}^{t}e(i) \\
&\hat{s}(t) = \sum_{i=0}^{t}\prod_{i<k\leq t}(1-\eta\lambda(k)) e(i) \\
&\hat{S}(t) = \sum_{i=0}^{t}\prod_{i<k\leq t}(1-\eta\lambda(k))e(i)\transpose XX\transpose \sum_{j=0}^{i-1}e(j).
\end{aligned}
\end{equation*}
\end{lemma}

\begin{proof}
We first write $W(t)$ in terms of $W(0)$ and $e(i)$, $i\in[t]$, so that
\begin{equation}
\label{eq:Wt}
    W(t) = W(0) -\frac{\eta}{\sqrt p} b\sum_{i=0}^{t-1}e(i)\transpose X = W(0) -\frac{\eta}{\sqrt p} bs(t-1)\transpose X.
\end{equation}
Similarly, for $\beta(t)$ we have
\begin{equation}
\label{eq:betat}
\begin{aligned}
    \beta(t) &=\prod_{i=0}^{t-1}(1-\eta\lambda(i))\beta(0)-\frac{\eta}{\sqrt p}\sum_{i=0}^{t-1}\prod_{i<k<t}(1-\eta\lambda(k))W(i)X\transpose e(i) \\
    &=\prod_{i=0}^{t-1}(1-\eta\lambda(i))\beta(0)-\frac{\eta}{\sqrt p}\sum_{i=0}^{t-1}\prod_{i<k<t}(1-\eta\lambda(k))\Big(W(0) -\frac{\eta}{\sqrt p} b\sum_{j=0}^{i-1}e(j)\transpose X\Big)X\transpose e(i) \\
    &=\prod_{i=0}^{t-1}(1-\eta\lambda(i))\beta(0)-\frac{\eta}{\sqrt p}\sum_{i=0}^{t-1}\prod_{i<k<t}(1-\eta\lambda(k))W(0)X\transpose e(i)\\
    &\quad + \frac{\eta^2}{p}b\sum_{i=0}^{t-1}\prod_{i<k<t}(1-\eta\lambda(k))e(i)\transpose XX\transpose \sum_{j=0}^{i-1}e(j) \\
    &=\prod_{i=0}^{t-1}(1-\eta\lambda(i))\beta(0) -\frac{\eta}{\sqrt p}W(0)X\transpose \hat{s}(t-1) +\frac{\eta^2}{p}b \hat{S}(t-1).
\end{aligned}
\end{equation}
We now study how the error $e(t)$ changes after a single update step, writing
\begin{equation*}
\begin{aligned}
    e(t+1) &= \frac{1}{\sqrt p}XW(t+1)\transpose \beta(t+1)-y \\
    &= \frac{1}{\sqrt p}X(W(t+1)-W(t)\transpose \beta(t+1)+ \frac{1}{\sqrt p}XW(t)\transpose (\beta(t+1)-(1-\eta\lambda(t))\beta(t)) \\
    &\quad +(1-\eta\lambda(t))\Big(\frac{1}{\sqrt p}XW(t)\transpose \beta(t)-y\Big) -\eta\lambda(t)y\\
    &=(1-\eta\lambda(t))e(t) - \frac{\eta}{p}b\transpose\beta(t+1)XX\transpose e(t)  - \frac{\eta}{p}XW(t)\transpose W(t)X\transpose e(t) - \eta\lambda(t)y
\end{aligned}
\end{equation*}
By plugging \eqref{eq:Wt} and \eqref{eq:betat} into above equation, we have
\begin{equation*}
\begin{aligned}
    e(t+1) &= (1-\eta\lambda(t))e(t) \\
    &\quad-\frac{\eta}{p}b\transpose\bigg[ \prod_{i=0}^{t}(1-\eta\lambda(i))\beta(0) -\frac{\eta}{\sqrt p}W(0)X\transpose \hat{s}(t) +\frac{\eta^2}{p}b \hat{S}(t)\bigg]XX\transpose e(t) \\
    &\quad - \frac{\eta}{p}X\bigg[W(0) -\frac{\eta}{\sqrt p}bs(t-1)\transpose X\bigg]\transpose \bigg[W(0) -\frac{\eta}{\sqrt p}bs(t-1)\transpose X\bigg]X\transpose e(t)\\
    &\quad - \eta\lambda(t)y
\end{aligned}
\end{equation*}
After expanding the brackets and rearranging the items, we can obtain \eqref{eq:err_update}.
\end{proof}

\begin{lemma}
\label{lma:ineqs_2}
Given $\delta\in(0,1)$ and $\epsilon>0$ , if $p=\Omega(\frac{1}{\epsilon}\log\frac{d}{\delta}+\frac{d}{\epsilon}\log\frac{1}{\epsilon})$, the following inequalities hold with probability at least $1-\delta$
\begin{equation}
\label{eq:bbeta_sqrtp}
\frac{|b\transpose\beta(0)|}{\sqrt p}\leq c\sqrt{\log\frac{1}{\delta}}
\end{equation}
\begin{equation}
\label{eq:bW_sqrtp}
\frac{\|b\transpose W(0)\|}{\sqrt p}\leq c\sqrt{d\log\frac{d}{\delta}}
\end{equation}
\begin{equation}
\label{eq:b_norm}
\Big|\frac{\|b\|^2}{p}-1\Big| \leq  \frac{c}{\sqrt p}\sqrt {\log \frac{1}{\delta}}
\end{equation}
\begin{equation}
\label{eq:ww_p}
\Big\|\frac{1}{p}W(0)\transpose W(0) - I_d \Big\|\leq \epsilon
\end{equation}
where $c$ is a constant.
\end{lemma}

\begin{proof}
\eqref{eq:bbeta_sqrtp} is derived from Lemma \ref{lem:inner-product-tail}. \eqref{eq:bW_sqrtp} is by \eqref{eq:bbeta_sqrtp} and a union bound argument. \eqref{eq:b_norm} is by Lemma \ref{lem:chi-squared-tail}. \eqref{eq:ww_p} is by Corollary \ref{cor:RIP}
\end{proof}

\begin{proof}[Proof of Theorem \ref{thm:lin_conv}]
We show \eqref{eq:reg_error_bd} by induction. Assume \eqref{eq:reg_error_bd} holds for all $t=0,1,...,t'$, we will show it hold for $t=t'+1$. For any $k\leq t'$, we apply \eqref{eq:reg_error_bd} repeatedly on the right hand side of itself to get
\begin{equation*}
    \|e(k)\|\leq \prod_{i=0}^{k-1}\Big(1-\frac{\eta\gamma}{2}-\eta\lambda(i)\Big)\|e(0)\| + \sum_{i=0}^{k-1}\eta\lambda(i)\prod_{i<j<k}\Big(1-\frac{\eta\gamma}{2}-\eta\lambda(j)\Big)\|y\|
\end{equation*}
For $t\leq t'$, we take the sum over $k=0,..,t$ on both sides of above inequality
\begin{equation*}
\begin{aligned}
    \sum_{k=0}^{t}\|e(k)\|
    &\leq \sum_{k=0}^{t}\prod_{i=0}^{k-1}\Big(1-\frac{\eta\gamma}{2}-\eta\lambda(i)\Big)\|e(0)\|  + \sum_{k=0}^{t}\sum_{i=0}^{k-1}\eta\lambda(i)\prod_{i<j<k}\Big(1-\frac{\eta\gamma}{2}-\eta\lambda(j)\Big)\|y\| \\
    &\leq \sum_{k=0}^{t}\Big(1-\frac{\eta\gamma}{2}\Big)^{k-1}\|e(0)\|+ \sum_{k=0}^{t}\sum_{i=0}^{k-1}\eta\lambda(i)\Big(1-\frac{\eta\gamma}{2}\Big)^{k-i-1}\|y\| \\
    &\leq \sum_{k=0}^{t}\Big(1-\frac{\eta\gamma}{2}\Big)^{k-1}\|e(0)\| + \eta\|y\|\sum_{k=0}^{t-1}\lambda(i)\sum_{k=i+1}^T\Big(1-\frac{\eta\gamma}{2}\Big)^{k-i-1} \\
    &\leq \frac{2}{\eta\gamma}\|e(0)\|+ \frac{2}{\gamma}S_\lambda\|y\| \\
    &\leq \frac{c\sqrt n}{\gamma}(\frac{1}{\eta}+S_\lambda)
\end{aligned}
\end{equation*}
where we use $\|e(0)\|=O(\sqrt n)$ and $\|y\|=O(\sqrt n)$. With this bound and the inequalities from Lemma \ref{lma:ineqs_2}, we can bound the norms of $J_1(t)$, $J_2(t)$ and $J_3(t)$ from Lemma \ref{lma:one_step_update}. It follows that
\begin{equation}
\label{eq:J1_bd}
\|J_1(t)\| \leq \frac{1}{p}|b\transpose\beta(0)|\|XX\transpose\|\leq c\frac{M\sqrt{\log \delta^{-1}}}{\sqrt p}\leq \frac{\gamma}{16},
\end{equation}
\begin{equation}
\label{eq:J2_bd}
\|J_2(t)\| \leq \frac{\eta}{p}\|X\| \|XX\transpose\|\|\bar{v}\|(2\|s(t-1)\|+\|\hat{s}(t)\|)\leq c\frac{\eta}{p}M^{3/2}\sqrt{d\log\frac{d}{\delta}}\frac{\sqrt n}{\gamma}(\frac{1}{\eta}+S_\lambda)\leq \frac{\gamma}{16}
\end{equation}
and
\begin{equation}
\label{eq:J3_bd}
\|J_3(t)\| \leq \frac{\eta^2}{p^2}\|b\|^2(\|XX\transpose\| |\hat{S}(t)|+\|XX\transpose\|^2 \|s(t-1)\|^2) \leq c\frac{\eta^2}{p}M^2 \frac{n}{\gamma^2}(\frac{1}{\eta}+S_\lambda)^2\leq \frac{\gamma}{16}
\end{equation}
hold for all $t\leq t'$ if $p = \Omega(\frac{Md\log(d/\delta)}{\gamma})$ and $S_\lambda = O(\frac{\gamma\sqrt{\gamma p}}{\eta\sqrt{n}M})$. Furthermore, since $\|\frac{1}{p}W(0)W(0)\transpose-I_d\|\leq \epsilon_0$ with high probability when $p=\Omega(d)$, we have
\begin{equation}
\label{eq:xwwx_bd}
\begin{aligned}
\|\frac{1}{p}XW(0)\transpose W(0)X\transpose-\gamma I_d\|
&\leq \|\frac{1}{p}XW(0)\transpose W(0)X\transpose-XX\transpose\| + \|XX\transpose-\gamma I_d\|  \\
&\leq (1+\epsilon)\epsilon_0\gamma+\epsilon\gamma \leq \frac{\gamma}{16}
\end{aligned}
\end{equation}
Therefore, combining \eqref{eq:J1_bd}, \eqref{eq:J2_bd}, \eqref{eq:J3_bd} and \eqref{eq:err_update}, we have
\begin{equation*}
\begin{aligned}
       \|e(t'+1)\|
       &\leq \Big(1-\eta\lambda(t')-\eta\gamma\Big) \|e(t')\| + \eta\Big\|\frac{\eta}{p}XW(0)\transpose W(0)X\transpose-\gamma I_d\Big\|\|e(t')\|\\
       &\quad + \eta(\|J_1(t')\|+\|J_2(t')\|+\|J_3(t')\|)\|e(t')\| + \eta\lambda(t')\|y\| \\
       &\leq \Big(1-\eta\lambda(t')-\eta\gamma\Big)\|e(t')\| + \frac{1}{16}\eta\gamma\|e(t')\|+ \frac{3}{16}\eta\gamma\|e(t')\|  + \eta\lambda(t')\|y\| \\
       &\leq\Big(1-\eta\lambda(t')-\frac{\eta\gamma}{2}\Big)\|e(t')\| +  \eta\lambda(t')\|y\|
\end{aligned}
\end{equation*}
which completes the proof.
\end{proof}

\begin{proof}[Proof of Proposition \ref{prop:isom}]
By Corollary \ref{cor:RIP}, if $d=\Omega(\frac{1}{\epsilon}\log\frac{n}{\delta}+\frac{n}{\epsilon}\log \frac{1}{\epsilon})$, we have
\begin{equation*}
\|XX\transpose - I_n\|\leq \epsilon
\end{equation*}
It follows that $\lambda_{\min}(XX\transpose)\geq 1-\epsilon$ and $\lambda_{\max}(XX\transpose)\leq 1+\epsilon \leq (1+4\epsilon)(1-\epsilon)$ for $\epsilon <1/2$.
\end{proof}

\begin{lemma}
\label{lma:suf_cond}
Recall from Lemma \ref{lma:one_step_update} that
\begin{equation*}
\beta(t)=\prod_{i=0}^{t-1}(1-\eta\lambda(i))\beta(0) -\frac{\eta}{\sqrt p}W(0)X\transpose \hat{s}(t-1) +\frac{\eta^2}{p}b \hat{S}(t-1)
\end{equation*}
with $\hat{s}(t) = \sum_{i=0}^{t}\prod_{i<k\leq t}(1-\eta\lambda(k)) e(i)$ and $\hat{S}(t) = \sum_{i=0}^{t}\prod_{i<k\leq t}(1-\eta\lambda(k))e(i)\transpose XX\transpose \sum_{j=0}^{i-1}e(j)$. Under the conditions of Theorem \ref{thm:lin_align}, if $t>C_1\frac{\log (p/\eta)}{\eta\lambda}$ and $\hat{S}(t)\geq \max(C_2\frac{\sqrt{p\gamma}}{\eta}\|\hat{s}(t)\|,1)$ for some positive constants $C_1$ and $C_2$, then $\cos\angle(b, \beta(t)) \geq c$ for some constant $c=c_\delta$.
\end{lemma}
\begin{proof}
We compute the cosine of the angle between $\beta(t)$ and $b$. With probability $1-\delta$,
\begin{equation*}
\begin{aligned}
    \cos\angle(b, \beta(t))
    &= \frac{b\transpose \beta(t)}{\|b\|\|\beta(t)\|}= \frac{\frac{b}{\|b\|}\transpose \beta(t)}{\|\beta(t)\|}\\
    &\geq \frac{\frac{\eta^2}{p}\|b\|\hat{S}(t-1) - (1-\eta\lambda)^t\|\beta(0)\| - \frac{\eta}{\sqrt p}\|\frac{b}{\|b\|}\transpose W(0)\|\|X\|\|\hat{s}(t-1)\|}{\frac{\eta^2}{p}\|b\|\hat{S}(t-1)+(1-\eta\lambda)^t\|\beta(0)\| + \frac{\eta}{\sqrt p}\|W(0)\|\|X\|\|\hat{s}(t-1)\|} \\
    &\geq \frac{c'_1\frac{\eta^2}{\sqrt p}\hat{S}(t-1) - c'_2\sqrt p(1-\eta\lambda)^t- c'_3\eta\sqrt{\frac{d\gamma}{p}}\|\hat{s}(t-1)\|}{c'_1\frac{\eta^2}{\sqrt p}\hat{S}(t-1) + c'_2\sqrt p(1-\eta\lambda)^t+ c'_4\eta\sqrt{\gamma}\|\hat{s}(t-1)\|}
\end{aligned}
\end{equation*}
where we use \eqref{eq:b_norm}, \eqref{eq:ww_p} and the tail bound for standard Gaussian vectors, and $c'_i$ are constants that only depend on $\delta$. Notice that if $t=\Omega(\frac{\log (p/\eta)}{\eta\lambda})$, we have $c'_2\sqrt p(1-\eta\lambda)^t = O(\frac{\eta^2}{\sqrt p})$. It follows that $\cos\angle(b, \beta(t)) \geq c$ if $\hat{S}(t-1)=\Omega(\frac{\sqrt{p\gamma}}{\eta}\|\hat{s}(t-1)\|+1)$.
\end{proof}

\begin{lemma}
\label{lma:decomp_1}
Consider the orthogonal decomposition $e(t) = a(t)\bar{y}+\xi(t)$, where $\bar{y}=-y/\|y\|$ and $\xi(t)\perp y$. Under the conditions of Theorem \ref{thm:lin_align}, there exists a constant $C_\tau>0$ such that for any $t\in[\tau,T]$ with $\tau=\frac{C_\tau}{\eta\lambda}$, we have
\begin{equation}
\label{eq:at_lbd}
a(t)\geq \frac{\lambda - \gamma}{\lambda+\gamma}\|y\|
\end{equation}
and
\begin{equation}
\label{eq:xit_ubd}
\|\xi(t)\|\leq \frac{\gamma}{\lambda+\gamma}\|y\|.
\end{equation}
\end{lemma}

\begin{proof}
By Theorem \ref{thm:lin_conv}, we have for all $t\leq T$, $\|e(t)\|\leq(1-\eta\lambda-\eta\gamma/2)\|e(t)\|+\eta\lambda\|y\|$. By rearranging the terms, we have
\begin{equation*}
    \|e(t+1)\| -\frac{\lambda }{\lambda - \gamma/2}\|y\| \leq (1-\eta\lambda-\frac{\eta\gamma}{2})\Big(\|e(t)\| -\frac{\lambda }{\lambda - \gamma/2}\|y\|\Big)
\end{equation*}
or
\begin{equation*}
\|e(t)\| -\frac{\lambda }{\lambda - \gamma/2}\|y\|\leq (1-\eta\lambda-\frac{\eta\gamma}{2})^t \Big(\|e_0\|- \frac{\lambda }{\lambda - \gamma/2}\|y\|\Big) \leq (1-\eta\lambda)^t(\|e_0\|+\|y\|).
\end{equation*}
Notice that $\|y\|$ and $\|e(0)\|$ are of the same order, so when $t\in [\tau_1,T]$ with $\tau_1 = \frac{c_1}{\eta\lambda}$ and some constant $c_1$, we have
\begin{equation}
\label{eq:et_y_bd}
\|e(t)\|\leq \frac{\lambda+\gamma/2}{\lambda-\gamma/2} \|y\|.
\end{equation}
In order to get a lower bound for $a(t)$, we multiply $\bar{y}\transpose$ on both sides of \eqref{eq:err_update}. It follows that for $t\in [\tau_1,T]$
\begin{equation*}
\begin{aligned}
a(t+1)
&\geq \bar{y}\transpose\Big(1-\eta\lambda-\eta\gamma\Big)e(t) - \eta\|\frac{1}{p}XW(0)\transpose W(0)X\transpose-\gamma I_d\|\|e(t)\|\\
&\quad-\eta(\|J_1(t)\|+\|J_2(t)\|+\|J_3(t)\|)\|e(t)\| +\eta\lambda\|y\|  \\
&\geq (1-\eta\lambda-\eta\gamma)a(t)- \frac{1}{4}\eta\gamma \|e(t)\| +\eta\lambda\|y\| \\
&\geq (1-\eta\lambda-\eta\gamma)a(t)+\frac{1}{2}\eta\gamma\|y\|.
\end{aligned}
\end{equation*}
In the second inequality, we use the bounds \eqref{eq:J1_bd}, \eqref{eq:J2_bd}, \eqref{eq:J3_bd} and \eqref{eq:xwwx_bd}. The last inequality is by \eqref{eq:et_y_bd} and $\lambda\geq 3\gamma$. Following a similar derivation, we have
\begin{equation*}
a(t) -\frac{\lambda - \gamma/2}{\lambda +\gamma}\|y\|\geq (1-\eta\lambda-\eta\gamma)^{t-\tau_1} \Big(a(\tau_1)- \frac{\lambda - \gamma/2}{\lambda +\gamma}\|y\|\Big) \geq -(1-\eta\lambda)^{t-\tau_1}(\|e(\tau_1)\|+\|y\|).
\end{equation*}
The bound \eqref{eq:at_lbd} holds when $t\in [\tau_1+\tau_2,T]$ with $\tau_2 = \frac{c_2}{\eta\lambda}$ and some constant $c_2$. Then we multiply $\frac{\xi(t+1)\transpose}{\|\xi(t+1)\|}$ on both sides of \eqref{eq:err_update}. This establishes that for $t\in [\tau_1,T]$
\begin{equation*}
\begin{aligned}
\|\xi(t+1)\|
&\leq \frac{\xi(t+1)\transpose}{\|\xi(t+1)\|}\Big(1-\eta\lambda-\eta\gamma\Big)e(t) + \eta\|\frac{1}{p}XW(0)\transpose W(0)X\transpose-\gamma I_d\|\|e(t)\|\\
&\quad+\eta(\|J_1(t)\|+\|J_2(t)\|+\|J_3(t)\|)\|e(t)\| +\eta\lambda\|y\|  \\
&\leq (1-\eta\lambda-\eta\gamma)\|\xi(t)\|+ \frac{\eta\gamma}{4}\|e(t)\| \\
&\leq (1-\eta\lambda-\eta\gamma)\|\xi(t)\|+ \frac{\eta\gamma}{2}\eta\gamma\|y\|.
\end{aligned}
\end{equation*}
The first inequality is by $\xi(t+1)\transpose y =0$ and in the second inequality we use $\xi(t+1)\transpose e(t) = \xi(t+1)\transpose\xi(t) \leq \|\xi(t+1)\|\|\xi(t)\|$. It follows that
\begin{equation*}
\|\xi(t)\|-\frac{\gamma/2}{\lambda+\gamma}\|y\|\leq (1-\eta\lambda-\eta\gamma)^{t-\tau_1}\Big(\|\xi(0)\|-\frac{\gamma/2}{\lambda+\gamma}\|y\|\Big)\leq (1-\eta\lambda)^{t-\tau_1}(\|e(\tau_1)\|+\|y\|).
\end{equation*}
The bound \eqref{eq:xit_ubd} holds when $t\in [\tau_1+\tau_3,T]$ with $\tau_3 = \frac{c_3}{\eta\lambda}$ for a constant $c_3$. Finally, the bounds \eqref{eq:at_lbd} and \eqref{eq:xit_ubd} hold when $t\in[\tau,T]$ with $\tau = \tau_1 + \max(\tau_2,\tau_3)$.
\end{proof}

\begin{lemma}
\label{lma:ST_sT}
Under the conditions of Theorem \ref{thm:lin_align}, suppose $T = \lfloor \frac{S_\lambda}{\lambda} \rfloor = C_T\frac{\sqrt{p}}{\eta\sqrt{n\gamma}}$. Then we have $\hat{S}(T) \geq \tilde{c}\frac{\sqrt {p\gamma}}{\eta}\|\hat{s}(T)\|$, where $C_T$ and $\tilde{c}$ are positive constants.
\end{lemma}
\begin{proof}
Notice that
\begin{equation*}
    e(i)\transpose XX\transpose e(j) \geq \gamma e(i)\transpose e(j) - \|e(i)\|\|e(j)\|\|XX\transpose-\gamma I\| \geq \gamma e(i)\transpose e(j) - \epsilon\gamma\|e(i)\|\|e(j)\|.
\end{equation*}
For $i\in[T/2,T]$ and $\tau$ defined in Lemma \ref{lma:decomp_1}, we have
\begin{equation}
\label{eq:eXXe}
\begin{aligned}
    e(i)\transpose XX\transpose \sum_{j<i}e(j)
    &= e(i)\transpose XX\transpose \sum_{\tau\leq j<i}e(j) + e(i)\transpose XX\transpose \sum_{j<\tau}e(j) \\
    &\geq \sum_{\tau\leq j<i}\Big(\gamma e(i)\transpose e(j) - \epsilon\gamma\|e(i)\|\|e(j)\|\Big) - 2\gamma \sum_{j<\tau}\|e(i)\|\|e(j)\| \\
    &\geq \sum_{\tau\leq j<i}\gamma\Big(a(i)a(j) - \|\xi(i)\|\|\xi(j)\| - \epsilon\|e(i)\|\|e(j)\|\Big) - 2c\tau\gamma \|y\|^2 \\
    &\geq (i-\tau)\gamma\Big[\Big(\frac{\lambda-\gamma}{\lambda+\gamma}\Big)^2\|y\|^2-\Big(\frac{\gamma}{\lambda+\gamma}\Big)^2\|y\|^2-\epsilon\Big(\frac{\lambda+\gamma/2}{\lambda-\gamma/2}\Big)^2\|y\|^2- \frac{2c\tau}{i-\tau}\|y\|^2\Big] \\
    &\geq \frac{T}{8}\gamma\|y\|^2 = \frac{C_T}{8}\frac{\sqrt{p}}{\eta\sqrt{n\gamma}}\gamma\|y\|^2 \\
    &\geq c\frac{\sqrt{p\gamma}}{\eta}\|y\|.
\end{aligned}
\end{equation}
The second inequality is the orthogonal decomposition of $e(i)$ and $\|e(i)\|\leq c\|y\|$ given by \eqref{eq:reg_error_bd}. The third inequality is by \eqref{eq:at_lbd}, \eqref{eq:xit_ubd} and \eqref{eq:et_y_bd} from Lemma \ref{lma:decomp_1}. The fourth inequality is by $\lambda=\Omega(\gamma)$, $i-\tau \geq T/4$ and the fact that $\tau/(i-\tau)$ is small ($p=\Omega(n)$). The last inequality is by $\|y\|=\Theta(\sqrt n)$. Therefore,
\begin{equation*}
\begin{aligned}
    \hat{S}(T) &= \sum_{i=0}^T(1-\eta\lambda)^{T-i}e(i)\transpose XX\transpose \sum_{j<i}e(j) \\
    &= \sum_{i=T/2}^T(1-\eta\lambda)^{T-i}e(i)\transpose XX\transpose \sum_{j<i}e(j) + (1-\eta\lambda)^{T/2}\sum_{i=0}^{T/2}(1-\eta\lambda)^{T/2-i}e(i)\transpose XX\transpose \sum_{j<i}e(j) \\
    &\geq \sum_{i=T/2}^T(1-\eta\lambda)^{T-i}c\frac{\sqrt{p\gamma}}{\eta}\|y\| + (1-\eta\lambda)^{T/2}\sum_{i=0}^{T/2}(1-\eta\lambda)^{T/2-i}c'T\gamma\|y\|^2 \\
    &\geq \frac{c}{2}\frac{\sqrt{p\gamma}}{\eta}\frac{\|y\|}{\eta\lambda} -(1-\eta\lambda)^{T/2}\frac{c'T\gamma\|y\|^2}{\eta\lambda} \\
    &\geq \frac{c}{4}\frac{\sqrt{p\gamma}}{\eta}\frac{\|y\|}{\eta\lambda}
\end{aligned}
\end{equation*}
where the last inequality is by $(1-\eta\lambda)^{T/2}\ll 1$ when $p=\Omega(n)$. On the other hand,
\begin{equation*}
\begin{aligned}
\|\hat{s}(T)\|
&\leq \sum_{i=0}^T(1-\eta\lambda)^{T-i}\|e(i)\|\leq \frac{c}{\eta\lambda}\|y\|.
\end{aligned}
\end{equation*}
Combining the above inequalities gives the proof.
\end{proof}

\begin{proof}[Proof of Theorem \ref{thm:lin_align}]
First, notice that $\lambda(t)=0$ when $t>T$. By Theorem \ref{thm:lin_conv} we have that the prediction error converges to zero exponentially fast, or $\|e(t+1)\|\leq (1-\eta\gamma/2)\|e(t)\|$. It follows that $\hat{S}(t)\to \hat{S}(\infty)$ and $\hat{s}(t)\to \hat{s}(\infty)$ as $t\to\infty$. By Lemma \ref{lma:suf_cond}, we know it suffices to show $\hat{S}(\infty)\geq C\frac{\sqrt{p\gamma}}{\eta}\|\hat{s}(\infty)\|$ with some constant $C$. Since
\begin{equation*}
\hat{S}(\infty) = \sum_{i=0}^\infty(1-\eta\lambda)^{(T-i)_+}e(i)\transpose XX\transpose \sum_{j<i} e(j) = \hat{S}(T)+ \sum_{i>T}e(i)\transpose XX\transpose \sum_{j<i}e(j)
\end{equation*}
and
\begin{equation*}
\hat{s}(\infty) = \sum_{i=0}^\infty(1-\eta\lambda)^{(T-i)_+}e(i)= \hat{s}(T)+ \sum_{i>T}e(i),
\end{equation*}
by Lemma \ref{lma:ST_sT}, it suffices to show
\begin{equation}
\label{eq:exxe_e_bd}
    \sum_{i>T}e(i)\transpose XX\transpose \sum_{j<i}e(j) \geq C\frac{\sqrt{p\gamma}}{\eta}\sum_{i>T}\|e(i)\|.
\end{equation}
We write $g = XX\transpose \sum_{j<T}e(j)$. Then we have
\begin{equation}
\label{eq:g_lbd}
\begin{aligned}
\|g\| &\geq \lambda_{\min}(XX\transpose) \Big[\Big\|\sum_{\tau\geq j<T}e(j)\Big\|-\sum_{j<\tau}\|e(j)\|\Big] \\
&\geq\lambda_{\min}(XX\transpose) \Big[\sum_{\tau\geq j<T}a(j)-\sum_{j<\tau}\|e(j)\|\Big] \\
&\geq \gamma\Big[(T-\tau)\Big(\frac{\lambda-\gamma}{\lambda+\gamma}\Big)\|y\| - \tau c\|y\|\Big]
\end{aligned}
\end{equation}
and
\begin{equation}
\label{eq:g_ubd}
\begin{aligned}
\|g\| &\leq \|XX\transpose\| \Big(\sum_{j<\tau}\|e(j)\| + \sum_{\tau\geq j<T}\|e(j)\|\Big) \\
&\leq (1+\epsilon)\gamma\Big[\tau c\|y\|+ (T-\tau)\Big(\frac{\lambda+\gamma/2}{\lambda-\gamma/2}\Big)\|y\|\Big]
\end{aligned}
\end{equation}
where we use the bounds \eqref{eq:at_lbd} and \eqref{eq:et_y_bd} from Lemma \ref{lma:decomp_1}. We further denote $\alpha(t) = \bar{g}\transpose e(t)$ where $\bar{g} = g/\|g\|$. Following the same calculation in \eqref{eq:eXXe}, we have
\begin{equation*}
\begin{aligned}
    g\transpose e(T)
    &= e(T)\transpose XX\transpose \sum_{j<T}e(j) \\
    &\geq (T-\tau)\gamma\Big[\Big(\frac{\lambda-\gamma}{\lambda+\gamma}\Big)^2\|y\|^2-\Big(\frac{\gamma}{\lambda+\gamma}\Big)^2\|y\|^2-\epsilon\Big(\frac{\lambda+\gamma/2}{\lambda-\gamma/2}\Big)^2\|y\|^2- \frac{2c\tau}{T-\tau}\|y\|^2\Big].
\end{aligned}
\end{equation*}
Then
\begin{equation*}
\begin{aligned}
\frac{\alpha(T)}{\|e(T)\|} &\geq \frac{g\transpose e(T)}{\|g\|\|e(T)\|} \\
&\geq \frac{(T-\tau)\gamma\Big[\Big(\frac{\lambda-\gamma}{\lambda+\gamma}\Big)^2\|y\|^2-\Big(\frac{\gamma}{\lambda+\gamma}\Big)^2\|y\|^2-\epsilon\Big(\frac{\lambda+\gamma/2}{\lambda-\gamma/2}\Big)^2\|y\|^2- \frac{2c\tau}{T-\tau}\|y\|^2\Big]}{(1+\epsilon)\gamma\Big[\tau c\|y\|+ (T-\tau)\Big(\frac{\lambda+\gamma/2}{\lambda-\gamma/2}\Big)\|y\|\Big]\times\Big(\frac{\lambda+\gamma/2}{\lambda-\gamma/2}\Big)\|y\|} \\
&\geq \frac{\Big[\Big(\frac{\lambda-\gamma}{\lambda+\gamma}\Big)^2-\Big(\frac{\gamma}{\lambda+\gamma}\Big)^2-\epsilon\Big(\frac{\lambda+\gamma/2}{\lambda-\gamma/2}\Big)^2- \frac{2c\tau}{T-\tau}\Big]}{(1+\epsilon)\Big[\frac{\tau c}{T-\tau}+ \Big(\frac{\lambda+\gamma/2}{\lambda-\gamma/2}\Big)\Big]\times\Big(\frac{\lambda+\gamma/2}{\lambda-\gamma/2}\Big)}.
\end{aligned}
\end{equation*}
Notice that $T/\tau = \Omega(\sqrt{p/n})$, so that when $p/n$, $\lambda/\gamma$ are large and $\epsilon$ is small, we have
\begin{equation}
\label{eq:alphaT_eT}
    \alpha(T)\geq \frac{3}{4}\|e(T)\|.
\end{equation}
In order to obtain the lower bound on $\alpha(t)$ for all $t\geq T$, we multiply $\bar{g}\transpose$ on both sides of \eqref{eq:err_update}. Notice $\lambda(t) = 0$ and apply the bounds \eqref{eq:J1_bd}, \eqref{eq:J2_bd}, \eqref{eq:J3_bd} and \eqref{eq:xwwx_bd}. We have that
\begin{equation*}
\begin{aligned}
    \alpha(t+1)
    &\geq (1-\eta\gamma)\bar{g}\transpose e(t) - \eta\|\frac{1}{p}XW(0)\transpose W(0)X\transpose-\gamma I_d\|\|e(t)\|\\
    &\quad-\eta(\|J_1(t)\|+\|J_2(t)\|+\|J_3(t)\|)\|e(t)\| \\
    &\geq (1-\eta\gamma)\alpha(t) - \frac{\eta\gamma}{4}\|e(t)\|
\end{aligned}
\end{equation*}
or for $t\geq T$,
\begin{equation}
    \alpha(t) \geq (1-\eta\gamma)^{t-T}\alpha(T) - \frac{\eta\gamma}{4}\sum_{i=T}^{t-1}(1-\eta\gamma)^{t-i}\|e(i)\|.
\end{equation}
Taking the sum over $t>T$, we have
\begin{equation}
\label{eq:sum_alphat}
\begin{aligned}
    \sum_{t>T}\alpha(t)
    &\geq \sum_{t>T}(1-\eta\gamma)^{t-T}\alpha(T) - \frac{\eta\gamma}{4}\sum_{t>T}\sum_{i=T}^{t-1}(1-\eta\gamma)^{t-i}\|e(i)\|  \\
    &\geq \frac{1-\eta\gamma}{\eta\gamma}\alpha(T) - \frac{\eta\gamma}{4}\sum_{i>T}\|e(i)\|\sum_{t>i}(1-\eta\gamma)^{t-i} \\
    &\geq \frac{1-\eta\gamma}{\eta\gamma} \Big(\alpha(T)-\frac{\eta\gamma}{4}\sum_{i>T}\|e(i)\|\Big) \\
    &\geq \frac{1-\eta\gamma}{\eta\gamma}(\alpha(T)-\frac{1}{2}\|e(T)\|) \\
    &\geq \frac{1-\eta\gamma}{4\eta\gamma}\|e(T)\|.
\end{aligned}
\end{equation}
The second inequality follows from switching the order of sums. The fourth inequality is by exponential convergence after $T$ steps. The last inequality is by \eqref{eq:alphaT_eT}. With the above inequalities, we are ready to bound the left hand side of \eqref{eq:exxe_e_bd}, obtaining
\begin{equation}
\label{eq:lhs_bound}
\begin{aligned}
    \sum_{i>T}e(i)\transpose XX\transpose \sum_{j<i}e(j)
    &=\sum_{i>T}e(i)\transpose XX\transpose \sum_{j<T}e(j) +\sum_{i>T}e(i)\transpose XX\transpose \sum_{j\geq T}e(j) \\
    &\geq \sum_{t>T}\alpha(t)\|g\|-2\gamma \Big(\sum_{i\geq t}\|e(i)\|\Big)^2 \\
    &\geq \frac{1-\eta\gamma}{4\eta\gamma}\|e(T)\|\gamma\Big[(T-\tau)\Big(\frac{\lambda-\gamma}{\lambda+\gamma}\Big)\|y\| - \tau c\|y\|\Big] - 2\gamma\frac{4}{\eta^2\gamma^2}\|e(T)\|^2\\
    &\geq \frac{1-\eta\gamma}{4\eta\gamma}\|e(T)\|\gamma\Big[(T-\tau)\Big(\frac{\lambda-\gamma}{\lambda+\gamma}\Big)\|y\| - \tau c\|y\| - \frac{64}{\eta\gamma(1-\eta\gamma)}\|y\|\Big] \\
    &\geq \frac{1-\eta\gamma}{4\eta\gamma}\|e(T)\|\gamma\frac{T}{2}\|y\|  = \frac{1-\eta\gamma}{4\eta\gamma} \|e(T)\|\gamma \frac{C_T}{2}\frac{\sqrt p}{\eta\sqrt{n\gamma}}\|y\|\\
    &\geq C\frac{1-\eta\gamma}{4\eta\gamma}\frac{\sqrt {p\gamma}}{\eta}\|e(T)\|.
\end{aligned}
\end{equation}
The second inequality is by \eqref{eq:sum_alphat} and \eqref{eq:g_lbd}. The third inequality is by $\|e(T)\|\leq 2\|y\|$. The last inequality is by $\|y\| = \Theta(\sqrt n)$. On the other hand,
\begin{equation}
\label{eq:rhs_bound}
\sum_{i>T}\|e(i)\| \leq \sum_{i>T}(1-\eta\gamma/2)^{i-T}\|e(T)\| =\frac{1-\eta\gamma/2}{\eta\gamma/2}\|e(T)\|
\end{equation}
Combining \eqref{eq:lhs_bound} and \eqref{eq:rhs_bound} implies \eqref{eq:exxe_e_bd}, as desired.
\end{proof}

%% file: lemmas.tex
\section{Technical Lemmas}

In this section, we list technical lemmas that are used in our proofs, 
with references. 
The first is a variant of the Restricted Isometry Property that bounds the spectral norm of a random Gaussian matrix around $1$ with high probability.

\begin{lemma}[\citealp{hand2018global}]\label{lem:RIP}
    Let $A\in \Rmn$ has \iid $\calN(0,1/m)$ entries. Fix $0<\eps<1$, $k < m$, and a subspace $T\subseteq \R^n$ of dimension $k$, then there exists universal constants $c_1$ and $\gamma_1$, such that with probability at least $1-(c_1/\varepsilon)^k e^{-\gamma_1\eps m}$,
    \begin{align*}
    (1-\eps)\|v\|^2_2 \leq \|Av\|^2_2 \leq (1+\eps)\|v\|^2_2, \quad \forall v\in T.
    \end{align*}
\end{lemma}
Let us take $k = n$ in \cref{lem:RIP} to get the following corollary.
\begin{corollary}\label{cor:RIP}
    Let $A\in \Rmn$ has \iid $\calN(0,1/m)$ entries. For any $0<\eps<1$, there exists universal constants $c_2$ and $\gamma_2$, such that with probability at least $1-(c_2/\eps)^d e^{-\gamma_2\eps m}$,
    \begin{align*}
        \|A\transpose A - I_m\|\leq \eps
    \end{align*}
\end{corollary}

Then following lemma gives tail bounds for $\chi^2$ random variables.
\begin{lemma}[\citealp{laurent2000adaptive}]\label{lem:chi-squared-tail}
    Suppose $X\sim \chi^2_p$, then for all $t\geq 0$ it holds
    \begin{align*}
        \prob\{X-p \geq 2\sqrt{pt} + 2t\} \leq e^{-t}
    \end{align*}
    and
    \begin{align*}
        \prob\{X-p \leq -2\sqrt{pt}\} \leq e^{-t}.
    \end{align*}
\end{lemma}

For two independent random Gaussian vectors, their inner product can be controlled with the following tail bound.
\begin{lemma}[\citealp{gao2020model}]\label{lem:inner-product-tail}
    Let $X,Y\in\Rp$ be independent random Gaussian vectors where $X_r\sim\calN(0,1)$ and $Y_r\sim\calN(0,1)$ for all $r\in[p]$, then it holds
    \begin{align*}
        \prob(|X\transpose Y| \geq \sqrt{2pt} + 2t) \leq 2e^t.
    \end{align*}
\end{lemma}

\paragraph{Broader impact and future directions.}

Understanding the mechanisms of neurocomputation serves as a vital step in understanding human cognition, and feedback alignment was proposed as such a biologically plausible model based on the highly successful backpropagation algorithm from machine learning; the theoretical guarantees that we provide in this paper also bring closer the methodology of machine intelligence to that of human intelligence, adding new possibilities to the study of neurocomutation. A possible step forward in this direction could be the integration of Hebbian learning, which is a local learning rule based on synaptic plasticity, the adaptation of brain neurons during the learning process \citep{hebb2005organization}. This simple but powerful rule has been widely studied in deep learning, meta learning, and reinforcement learning \citep{mesnard2016towards,najarro2020meta,mahmoudi2013towards}. Moreover, it is also important to explore the intersection of neurocomputation with human cognition and behavior. For instance, existing works have studied methods of risk-sensitive reinforcement learning, which are able to produce agents with different risk preferences \citep{shen2014risk,fei2020risk,fei2021risk,fei2021exponential}, and neuroscientific experiments also pointed out the similarities between the strategies of the trained agents and human behavior under a variety of risk levels \citep{niv2012neural,shen2014risk}. Current algorithms rely on explicitly crafted utility functions with hand-tuned risk parameters, and model could be more biological plausible if the risk preference could be incorporated into the learning rules of deep neural networks for applications in deep reinforcement learning.